\documentclass[11pt]{article}
\usepackage{fullpage}
\usepackage{amsthm}        
\usepackage{amsmath}
\usepackage{amssymb}
\usepackage[numbers]{natbib}
 
 \usepackage{mathtools}  
\mathtoolsset{showonlyrefs} 
 
\usepackage[T1]{fontenc}
\usepackage{microtype}      %
\usepackage{xcolor}         %

 \usepackage[colorlinks=true, allcolors=black,c itecolor=blue, urlcolor=black]{hyperref}
\usepackage{amsfonts}       %

\usepackage{nicefrac}       %
\usepackage{wasysym}        %

\usepackage{subcaption} 
\usepackage{tikz}
\usetikzlibrary{arrows.meta, arrows, shapes, positioning}
\usepackage{pgfplots}
\pgfplotsset{compat=1.18} %
\usepackage{bchart}

\usepackage{booktabs} 
\usepackage{algorithm}
\usepackage{algorithmic}

\usepackage{url}            
\usepackage{enumitem}
\usepackage{xurl}

\usepackage[textsize=tiny]{todonotes}
\usepackage{tabularx, makecell}   %
\tikzstyle{main node}=[draw]

\definecolor{wjs}{RGB}{200,0,50}

\newtheorem{theorem}{Theorem}
\newtheorem{assumption}{Assumption}
\newtheorem{lemma}[theorem]{Lemma}

\newtheorem{definition}[theorem]{Definition}

\def\sT{{\mathsf T}}


\usepackage{lastpage} %

\title{Recommending Best Paper Awards for ML/AI Conferences via the Isotonic Mechanism}
 
\author{
  { Garrett G. Wen \textsuperscript{1}}\qquad
  {Buxin Su\textsuperscript{2}}  \qquad
  {Natalie Collina\textsuperscript{2}} \qquad
 {Zhun Deng\textsuperscript{3}}  \qquad
  {Weijie Su\textsuperscript{2}}\\[1ex]
  \textsuperscript{1}Yale University\\
   \textsuperscript{2}University of Pennsylvania\\
   \textsuperscript{3}University of North Carolina at Chapel Hill\\[2ex]
}
\newcommand\blfootnote[1]{%
  \begingroup
  \renewcommand\thefootnote{}%
  \footnote{\hspace{-1.2em}#1}
  \addtocounter{footnote}{-1}%
  \endgroup
}
\begin{document}

\maketitle

\begin{abstract}
Machine learning and artificial intelligence conferences such as NeurIPS and ICML now regularly receive tens of thousands of submissions, posing significant challenges to maintaining the quality and consistency of the peer review process. This challenge is particularly acute for best paper awards, which are an important part of the peer review process, yet whose selection has increasingly become a subject of debate in recent years. In this paper, we introduce an author-assisted mechanism to facilitate the selection of best paper awards. Our method employs the Isotonic Mechanism for eliciting authors' assessments of their own submissions in the form of a ranking, which is subsequently utilized to adjust the raw review scores for optimal estimation of the submissions' ground-truth quality. We demonstrate that authors are incentivized to report truthfully when their utility is a convex additive function of the adjusted scores, and we validate this convexity assumption for best paper awards using publicly accessible review data of ICLR from 2019 to 2023 and NeurIPS from 2021 to 2023. Crucially, in the special case where an author has a single \emph{quota}---that is, may nominate only one paper---we prove that truthfulness holds even when the utility function is merely nondecreasing and additive. This finding represents a substantial relaxation of the assumptions required in prior work. For practical implementation, we extend our mechanism to accommodate the common scenario of overlapping authorship. Finally, simulation results demonstrate that our mechanism significantly improves the quality of papers selected for awards.
\blfootnote{Emails:\;\texttt{gang.wen@yale.edu},\;\texttt{subuxin@sas.upenn.edu},\;\texttt{ncollina@seas.upenn.edu},\;\texttt{zhundeng@cs.unc.edu}, \;\texttt{suw@wharton.upenn.edu}} 
\end{abstract}

\section{Introduction}

Over the last few years, machine learning (ML) and artificial intelligence (AI) conferences such as ICML and NeurIPS have experienced an explosion of submissions. For example, NeurIPS received 3,240 submissions in 2017, which rose to 15,671 in 2024.
For 2025,   the count has increased to 21,575 submissions \citep{chairs2025Reflections2025Review2025}---
 a scale that threatens to overwhelm the current reviewing infrastructure if no structural reforms are enacted. This unprecedented explosion, in conjunction with the limited pool of experienced reviewers, has been gradually degrading the quality of peer review at these ML/AI conferences. Review scores have become increasingly noisy and exhibit a worrisome level of arbitrariness \citep{brezisArbitrarinessPeerReview2020, cortesInconsistencyConferencePeer2021, cheah2022should, beygelzimer2023has, langford2015arbitrariness, yuan2022can}. The NeurIPS 2014 experiment, which revealed substantial inconsistency when the same papers were reviewed by different committees, starkly illustrated the extent of this problem \citep{langford2015arbitrariness}. This instability could be further worsened by the recent OpenReview identity leak \citep{OpenReview.net_2025}, which undermines anonymity and likely adds even more noise to the scores as reviewers are becoming overly cautious.

Beyond general review quality concerns, the selection of best papers has emerged as a particularly acute problem. Best paper awards carry immense prestige and significantly impact researchers' careers and institutional rankings. Unfortunately, the selection process has sparked major controversies in recent years \citep{carlini2022no,Orabona2023Fiasco,Wired2024NeurIPSControversy}.
These controversies threaten the credibility of awards that have long served as markers of excellence in the field.

To address these critical challenges, we introduce an author-assisted mechanism for \textit{best paper selection}. Our method elicits authors' assessments of their own submissions through rankings, which represent a set of pairwise preferences, and are subsequently utilized to recalibrate raw review scores for optimal estimation of submissions' ground-truth quality. We demonstrate that authors are incentivized to report truthfully when their utility is a nondecreasing convex additive function of the adjusted scores. Importantly, when an author may nominate only one paper, truthfulness holds under the substantially weaker assumption that the utility is merely additive and nondecreasing—a significant relaxation that eliminates the restrictive convexity requirement imposed by prior work \citep{Su2021YouAT, suTruthfulOwnerAssistedScoring2022a, yanIsotonicMechanismExponential2023a, wuIsotonicMechanismOverlapping2023}.

Our new approach builds upon and extends the applicability of the Isotonic Mechanism framework \citep{Su2021YouAT, suTruthfulOwnerAssistedScoring2022a}. Among numerous efforts to address review quality challenges \citep{kobren2019paper, wang2019your, jecmen2020mitigating, leyton2022matching, stelmakh2023large,neurips2024assignment,neurips2024llm}, the Isotonic Mechanism offers a resource-free approach that incorporates authors’ evaluations of their own papers without imposing additional burden on reviewers.
 In its standard formulation, when an author submits $n \ge 2$ papers, she provides a ranking $\pi$ of these submissions according to her perception of the relative quality of the submissions before receiving review scores $\boldsymbol{y} = (y_1,y_2,\ldots,y_n)$. The mechanism then obtains adjusted scores, denoted as $\boldsymbol{\hat{R}}=(\hat{R}_1,\hat{R}_2,\ldots,\hat{R}_n)$, by solving:
\begin{equation}\label{equ:iso}
\begin{array}{cl}
    \min _{\boldsymbol{r}} & \|\boldsymbol{y}-\boldsymbol{r}\|^{2} \\
    \text{s.t.} & r_{\pi(1)}\ge r_{\pi(2)} \ge \cdots \ge r_{\pi(n)},
\end{array}
\end{equation}
where $\|\cdot\|$ denotes the $\ell_2$ norm. \cite{su2025icml} demonstrates that adjusted scores more effectively reflect the true quality of papers, while \cite{su2025find} shows that authors’ rankings correlate more strongly with citation counts and GitHub stars than review scores do.

The truthfulness and accuracy of the Isotonic Mechanism rely on two key assumptions regarding the noise model and authors' utility functions \citep{suTruthfulOwnerAssistedScoring2022a}. \cite{suTruthfulOwnerAssistedScoring2022a} first assumed that the review scores $\boldsymbol{y}$ are unbiased with respect to the true scores $\boldsymbol{R}$, where $\boldsymbol{R}$ can be practically regarded as the average scores assigned by a very large number (e.g., 1,000) of reviewers.

\begin{assumption}\label{Assumption1}
    The noisy score vector $\boldsymbol{y}$ is obtained by adding exchangeable noise to the true scores $\boldsymbol{R}$. That is,  
    \begin{equation}
        \boldsymbol{y} = \boldsymbol{R} + \boldsymbol{\epsilon},
    \end{equation}  
    where $\boldsymbol{\epsilon} = (\epsilon_1, \ldots, \epsilon_n)$ follows the same probability distribution as $(\epsilon_{\pi(1)}, \ldots, \epsilon_{\pi(n)})$ for any permutation $\pi$.
\end{assumption}  

\cite{suTruthfulOwnerAssistedScoring2022a} further assumes the authors' overall utility is an additive sum of some nondecreasing convex function.
\begin{assumption}\label{Assumption2}
    The ultimate goal of a rational author is to maximize her expected overall utility, where the overall utility is the additive sum of some nondecreasing convex function of the review scores:
    \begin{equation}
        U(\boldsymbol{\hat{R}}) = U({\hat{R}_1})+U({\hat{R}_2})+\cdots+U({\hat{R}_n}).
    \end{equation}
\end{assumption}
Recent work \citep{yanIsotonicMechanismExponential2023a} has relaxed \textbf{Assumption \ref{Assumption1}}, allowing for a broader class of noise models, such as those drawn from an exponential family distribution. Even in these settings, the Isotonic Mechanism retains its key advantages. In contrast, the chief obstacle to practical adoption now lies in \textbf{Assumption~\ref{Assumption2}}, whose convexity clause is difficult to justify empirically because real authors' utilities are heterogeneous and only partially observable. Accordingly, this paper concentrates on \emph{verifying} and \emph{relaxing} the convexity requirement. Exchangeability and other noise-model questions are undeniably important, but a full treatment of those issues is beyond our present scope.
We propose two key contributions that address the convexity challenge and enhance the practical applicability of author-assisted mechanisms for best paper selection:

 \paragraph{Relaxing convexity:} We explore conditions under which the convexity requirement can be dropped. Notably, when restricting attention to an author's  best paper, utility functions need only be nondecreasing rather than convex. This insight leads us to propose a best paper recommendation procedure that, based on our experiments, improves the accuracy of high-quality paper selection.   Crucially, this theoretical relaxation  turns our methodology into a robust framework for general best item selection problems, as we detail in Section~\ref{sec:contributions}.

 \paragraph{Validating convexity:} We investigate whether the convexity assumption can be preserved and propose empirical methods to verify it. By focusing on best papers, we obtain a natural and plausible interpretation of the utility function that ensures the truthfulness of the Isotonic Mechanism. Furthermore, this interpretation allows for empirical verification of convexity.  
\\ \vskip -0.1in
From a practical standpoint, the Isotonic Mechanism has been piloted at ICML from 2023 to 2025 \citep{OpenRank}, and its underlying principles have shown strong empirical backing. For instance, analysis of the ICML 2023 experiment provides positive evidence for improving review quality \citep{su2025icml}. Furthermore, authors' internal rankings, which are a key input to the mechanism, have been shown to correlate more strongly with long-term impact metrics like citation counts and GitHub stars than review scores do \citep{su2025find}.  It is recently formally adopted by ICML 2026 as a quality control signal to flag submissions with significant discrepancies between review scores and author rankings \citep{kamathIntroducingICML20262026}. Our extension to best paper selection is the natural next step: it both relaxes the convexity assumption that has hindered practical adoption, while retaining the original method’s resource-free nature.

\subsection{Best Paper Recommending Procedure}
\label{sec:bestpaper}
Let us now briefly describe how the Isotonic Mechanism contributes to best paper recommendation. Theoretical justifications are provided in Section \ref{sec:main}, while the procedure and its empirical performance are discussed in Section \ref{prod}. Importantly, the Isotonic Mechanism solely adjusts peer review scores based on authors’ rankings and serves as a reference point. The final decision on best paper selection remains with the conference organizers, whose choices may also influence authors’ utility functions.

Before proceeding, a key consideration is determining the \emph{quota}---the number of papers each author may submit for best paper review. In practice, a conference typically involves $M$ authors, and the \emph{quota} can vary from one author to another. Nevertheless, as long as the utility functions satisfy the conditions in Theorem~\ref{thm:betteru}, our approach remains valid for each individual author. This naturally leads to the question: How should we select the \emph{quota}? Below, we consider two possible approaches:

\paragraph{Uniform quota:} Every author is allowed to nominate the same number of papers (equal to the \emph{quota}), or all their papers if they have fewer than the \emph{quota}. This method promotes fairness and diversity but may put prolific authors, who have multiple high-quality papers, at a disadvantage.

\paragraph{Variable quota:} Authors with more submissions are permitted to nominate a greater number of papers, thus reflecting their broader research output. While this approach recognizes higher overall productivity, it can raise concerns about fairness and introduces additional complexity.
 \\ \vskip -0.1in
In this paper, we primarily analyze the \textbf{uniform quota} setting. Experiments in Appendix~\ref{app:fullexp} suggest that increasing the \emph{quota} has very little or even negative impact; in most conferences, best paper selections rarely feature multiple papers from the same author. Consequently, a \textbf{uniform quota} of 1 (or, more conservatively, 2 or 3) typically suffices. Nevertheless, the ultimate decision regarding the \emph{quota} (whether uniform or variable) rests with the conference organizers, who must balance the distinct goals of fairness, inclusiveness, and diversity.
 
We now apply the Isotonic Mechanism to adjust scores for each author’s papers. Since best paper selection is independent of paper acceptance, we focus only on accepted papers. Each author’s top $k$ papers (or all, if the total is fewer than $k$) are considered for best paper review, where $k$ denotes the \emph{quota}. We examine two scenarios:

\paragraph{Blind Case:} The decision-maker observes both the original review scores and the adjusted scores, but has no access to the authors’ internal rankings. Because the review process is double-blind, there is no way to infer how each author personally ranks their own papers.
\paragraph{Informed Case:} The decision-maker has full knowledge of each author’s $k$ selections (i.e., which papers they deem their best), including their relative rankings (which may include ties), and incorporates this information into the best paper selection process. 
\\\vskip -0.1in
One might wonder why we focus on these two cases. Consider a teacher choosing the best student essays: she might evaluate all submissions without knowledge of the students’ personal preferences (\textbf{Blind Case}) or allow students to indicate which essays they believe are their strongest (\textbf{Informed Case}). While the first scenario removes bias that could stem from self-assessment, it may also overlook the work students value most. By contrast, the second scenario captures students’ preferences but risks introducing bias. A similar trade-off applies to best paper selection: the process is inherently subjective, influenced by factors such as perceived research quality, author reputation, and other external factors. Studying both cases offers a broader perspective on the decision-making process.

\bigskip

\noindent
\textbf{Utility Functions and Truthfulness.} The fundamental distinction between these two cases lies in the author’s utility function.\footnote{Here we consider the perspective of a single rational author. The mechanism's extension to the multi-ownership setting, where co-authors may provide different rankings, is detailed in Section \ref{sim}.} Because an author’s utility depends solely on their top $k$ papers (where $k \le n$), we adopt a \emph{separable} utility function, as in \citep{suTruthfulOwnerAssistedScoring2022a}:
\begin{equation} 
	U(\boldsymbol{\hat{R}}) \;=\; U_1(\hat{R}_{(1)}) + U_2(\hat{R}_{(2)}) \;+\; \cdots \;+\; U_k(\hat{R}_{(k)}).
\end{equation}
Here, $\hat{R}_{(1)} \ge \hat{R}_{(2)} \ge \cdots \ge \hat{R}_{(n)}$ denotes the adjusted scores sorted from highest to lowest. The specific form of the functions $\{U_i\}$ distinguishes our two main scenarios, the \textbf{Blind} and \textbf{Informed} cases, which we detail next.
\paragraph{Blind Case.} Under the \textbf{Blind Case}, each paper is treated symmetrically. Consequently, it is reasonable to assume that $U_1 = U_2 = \cdots = U_k$, giving:
\begin{equation}
	U(\boldsymbol{\hat{R}}) \;=\; U(\hat{R}_{(1)}) \;+\; U(\hat{R}_{(2)}) \;+\; \cdots \;+\; U(\hat{R}_{(k)}).
\end{equation}

\paragraph{Informed Case.} Decision-makers see each author’s ranking for their top $k$ papers, so it is natural to assume that a higher-ranked paper offers greater potential utility to the author. In other words, $U_1, U_2, \dots, U_k$ may differ from one position to the next, thus preserving the index.
\\ \vskip -0.1in
Regardless of which case applies, a rational author seeks to maximize her expected utility, $\mathbb{E}[U(\boldsymbol{\hat{R}})]$. \emph{Truthfulness} means that to achieve this, the author will always report her true paper ranking. Formally, if an author’s true ranking is $R_1 \ge R_2 \ge \cdots \ge R_n$, then for any non-trivial permutation $\pi$ of $1, 2, \ldots, n$ and the corresponding isotonic solution $\boldsymbol{\hat{R}_\pi}$, we have:
\begin{equation}
	\mathbb{E}\bigl[ U(\boldsymbol{\hat{R}}) \bigr] \;\ge\; \mathbb{E}\bigl[ U(\boldsymbol{\hat{R}_\pi}) \bigr].
\end{equation}

In this paper, we establish the truthfulness conditions shown in Table~\ref{tab:prac}, demonstrating that under these utility assumptions, authors are always incentivized to be truthful.

When the \emph{quota} is one (i.e., each author can only nominate a single paper), only a \emph{nondecreasing} utility function is required for truthfulness, which is intuitive and notably weaker than convexity. Under this more permissive condition, all rational authors remain truthful.
In reality, authors often contribute to multiple papers, and each paper may have several authors, each with their own utility function. Building on the insights of \citep{wuIsotonicMechanismOverlapping2023}, Section~\ref{prod} discusses how the Isotonic Mechanism can be extended to more complex scenarios involving overlapping authorship. Full details of a synthetic yet realistic simulation are provided in Section~\ref{app:experiment}, demonstrating that the Isotonic Mechanism can be effectively applied to recommend best papers, achieving a significant improvement in the accuracy of high-quality paper selection.

\subsection{A Natural Interpretation of the Utility Function: Validating Convexity}

It is worth noting that, apart from the special case where the \emph{quota} is 1, we typically require the utility function to be convex to ensure the truthfulness of a rational author. Some authors may indeed perceive their marginal utility as increasing---for instance, if having one standout paper greatly enhances their overall influence---which would make their utility appear convex. Nevertheless, utility is inherently subjective, and unless an author naturally views it in such a manner, assuming convexity remains challenging.

Alternatively, we can interpret the utility function more objectively by viewing it as the sum of conditional probabilities that an author's paper is chosen as ``best,'' given the final adjusted scores:
\begin{equation}\label{equ:natu}
	U(\boldsymbol{\hat{R}}) \;\sim\; \sum_{i=1}^k \mathbb{P}\bigl(\text{Chosen as best}\;\big|\;\hat{R}_{(i)}\bigr),
\end{equation}
where $k$ is the \emph{quota}.

Intuitively, this interpretation of the utility function captures the overall probability of an author's paper being selected as the best paper. The challenge now lies in the fact that in private, the original utility function was (assumed to be) known to the author, whereas this interpretation of the utility function, defined as a sum of the conditional probabilities, is no longer known. But given that the role of the utility function in the Isotonic Mechanism does not depend on its specific form, we can show that leveraging historical conference data from ICLR 2019---2023 \citep{ICLR2019a, syncedICLR2020Virtual2020a, iclrAnnouncingICLR2021a, brockmeyerAnnouncingICLR20222022a, wangAnnouncingICLR20232023} and NeurIPS~2021---2023 \citep{chairsAnnouncingNeurIPS20212021, chairs2023AnnouncingNeurIPS20222022, chairs2023AnnouncingNeurIPS20232023} enables us to assess whether the resulting utility function is convex. In essence, although the original (private) utility remains unknown, empirical evidence can suggest whether the conditional probability is convex in practice. 

\begin{figure*}[ht]
\centering
\begin{subfigure}{0.45\textwidth}
\includegraphics[width=1.2\textwidth]{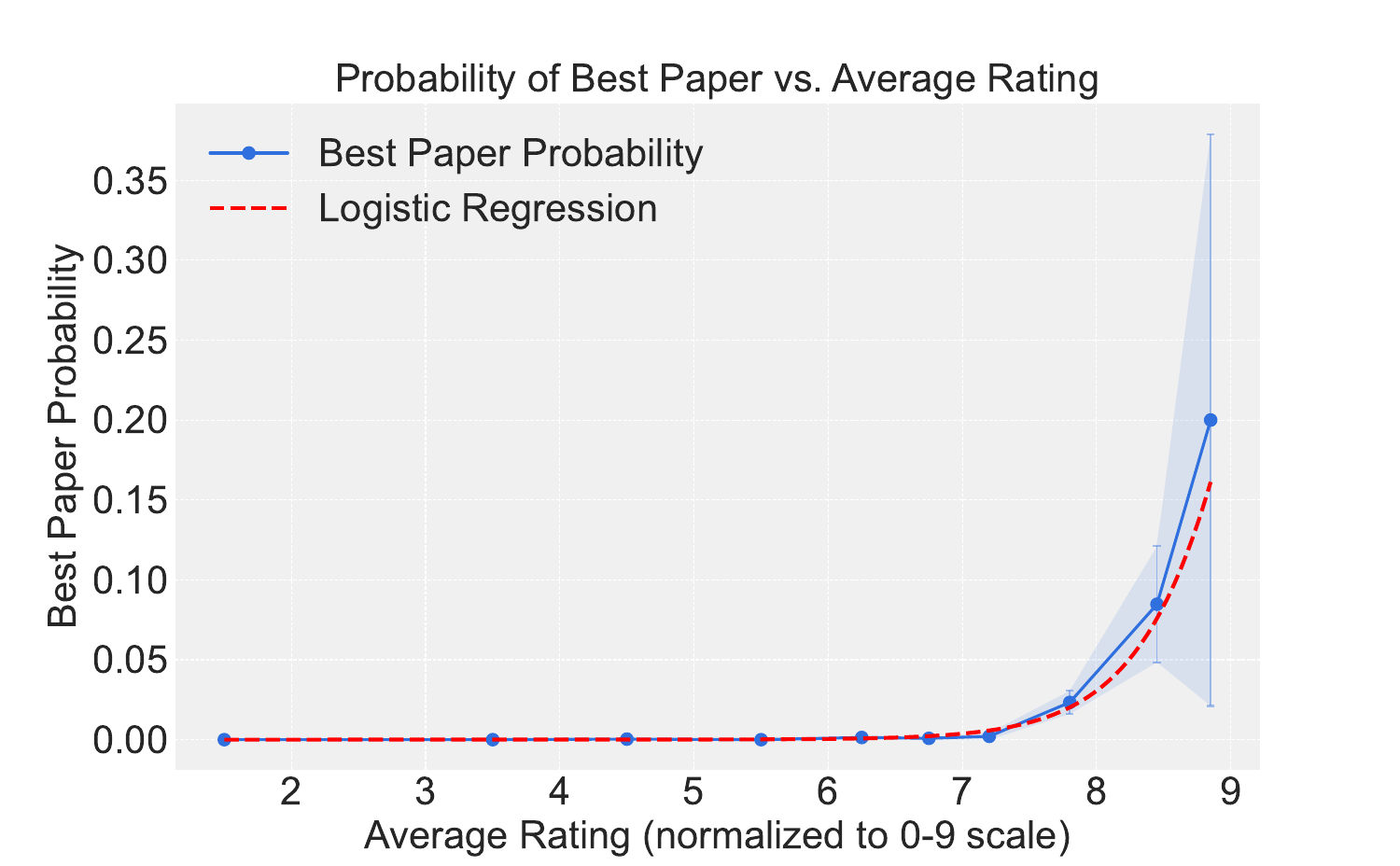}
\caption{ Best Paper Probabilities of ICLR 2019---2023}
\label{fig:bestlog}
\end{subfigure}
\hfill
\begin{subfigure}{0.45\textwidth}
\includegraphics[width=1.2\textwidth]{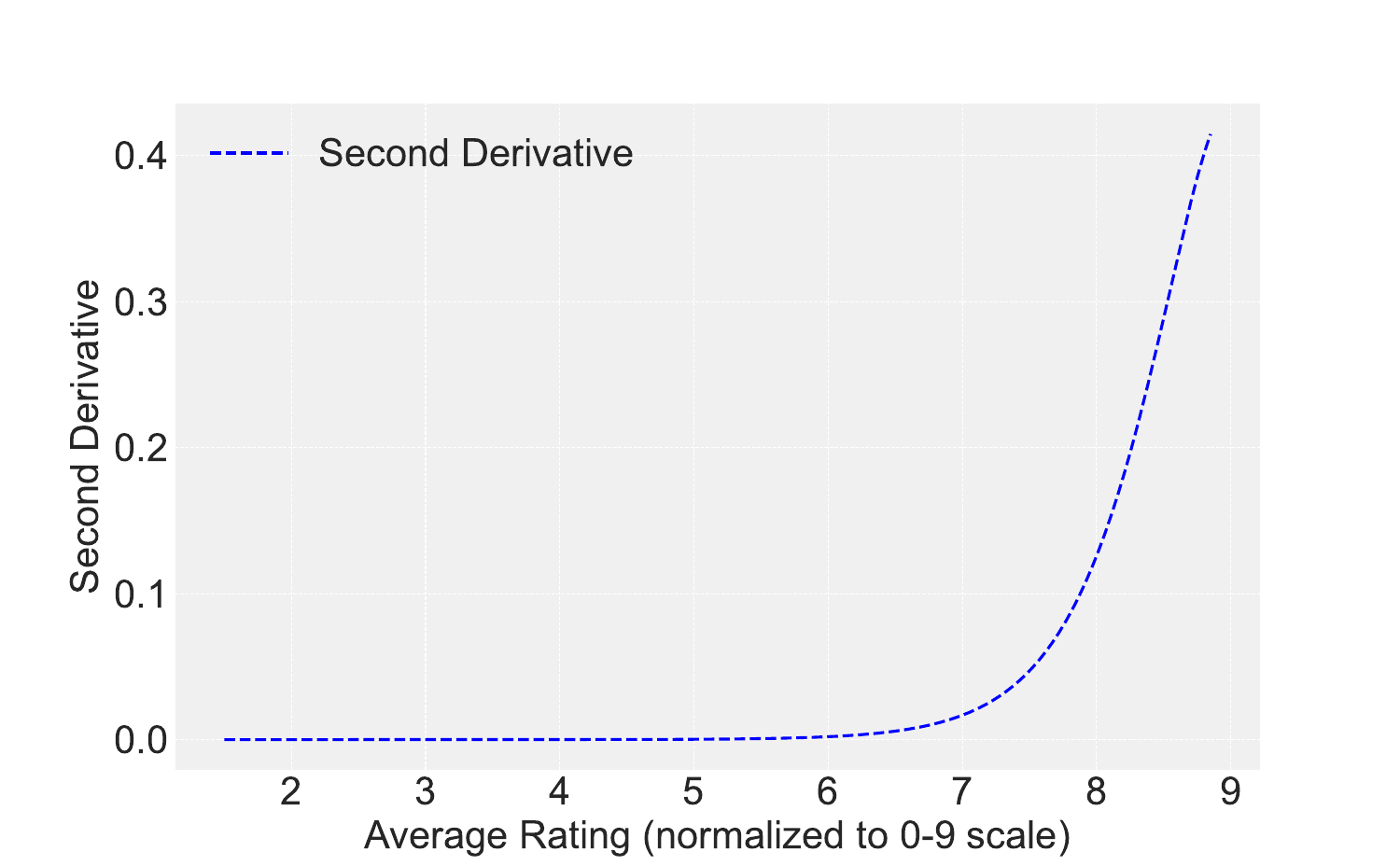}
\caption{Second Derivative of Regression of Best Paper Probabilities of ICLR 2019---2023}
\label{fig:second1}
\end{subfigure}
\hfill
\begin{subfigure}{0.45\textwidth}
\includegraphics[width=1.2\textwidth]{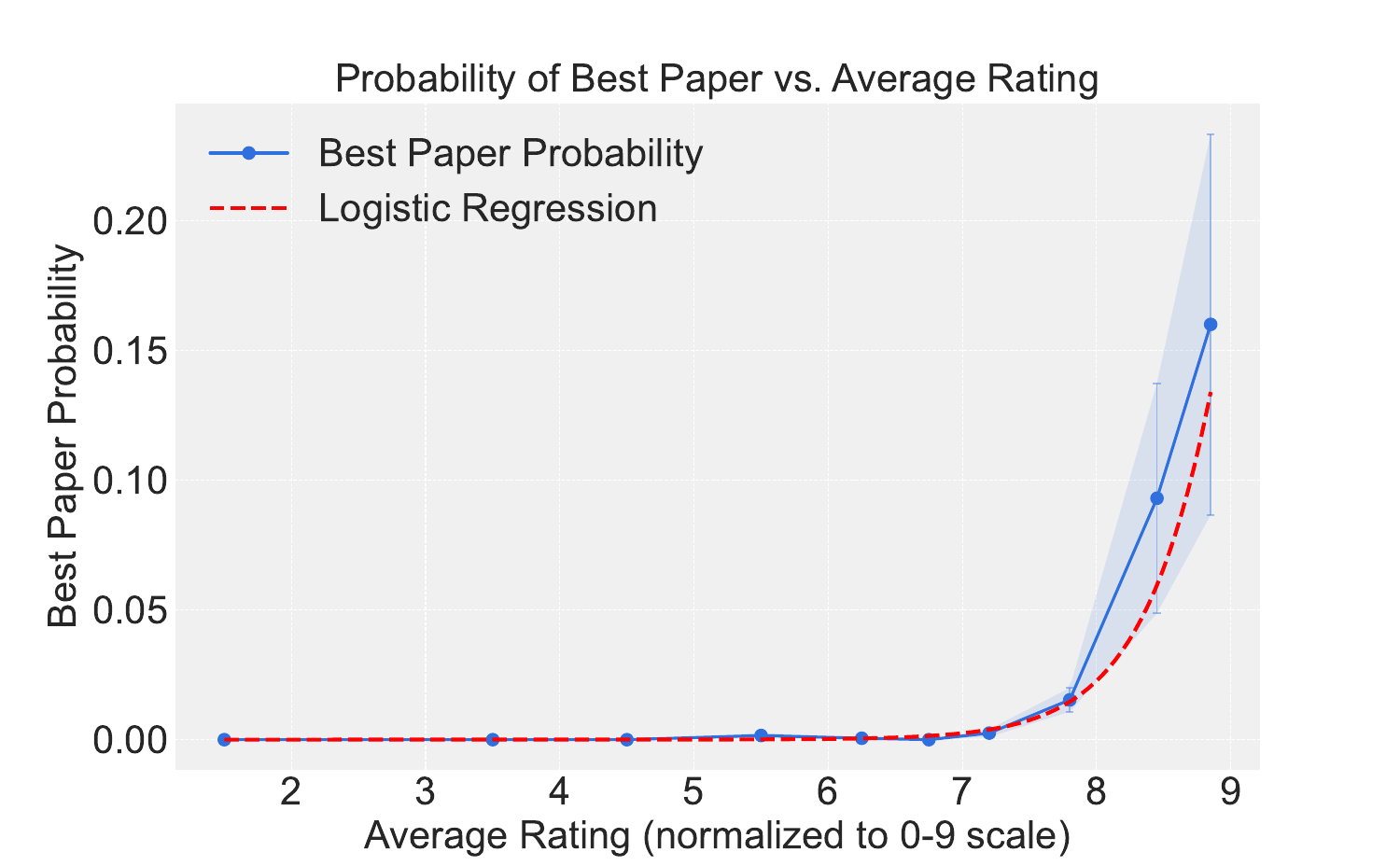}
\caption{ Best Paper Probabilities of NeurIPS 2021---2023}
\label{fig:bestlog4}
\end{subfigure}
\hfill
\begin{subfigure}{0.45 \textwidth}
\includegraphics[width=1.2\textwidth]{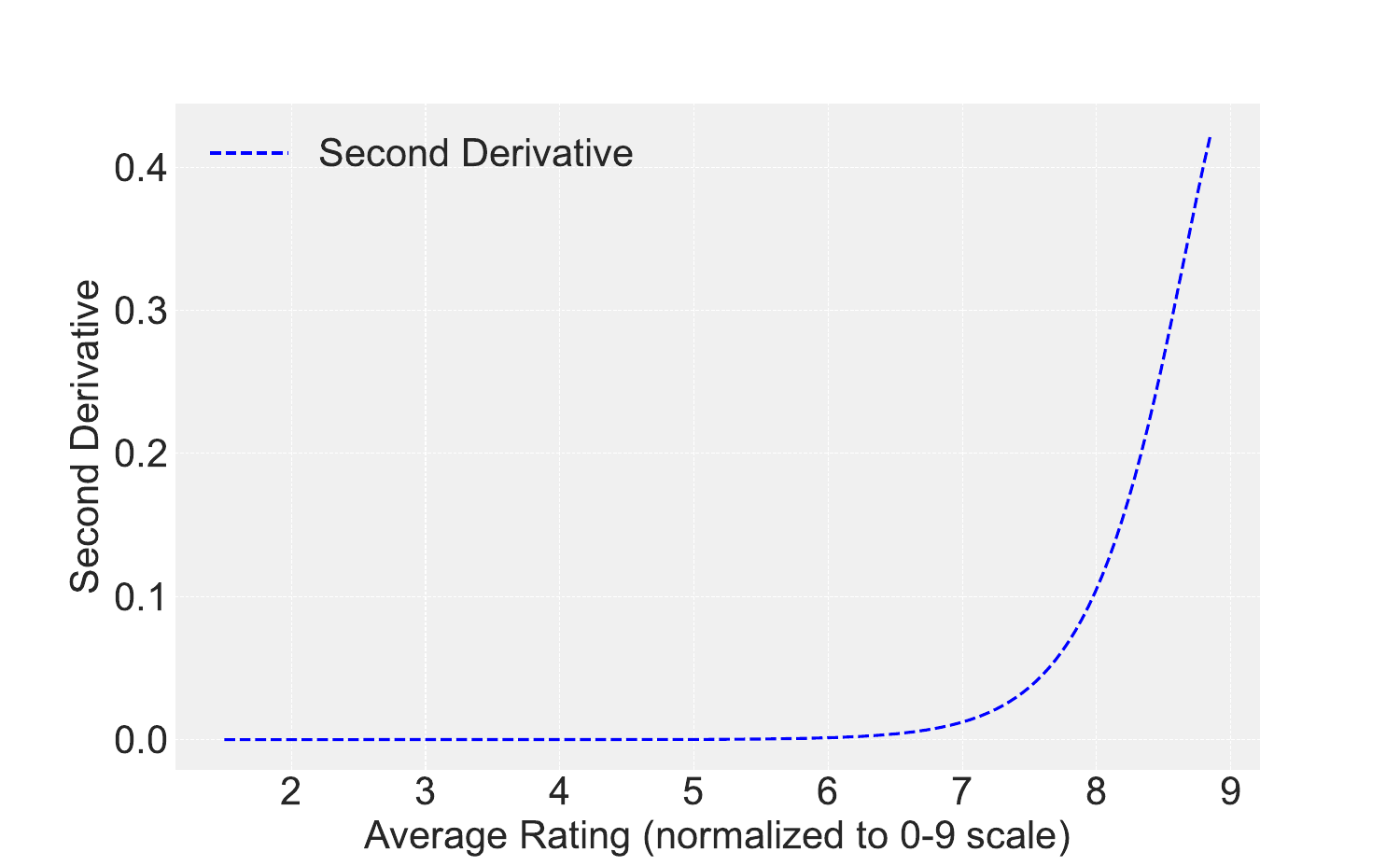}
\caption{Second Derivative of Regression of Best Paper Probabilities of NeurIPS 2021---2023}
\label{fig:second2}
\end{subfigure}
\caption{
Comparison and second derivative of the regression of best paper probabilities across ICLR~2019---2023 and NeurIPS~2021---2023. See Section~\ref{sec:expe} and Section~\ref{exp:NeurIPS} for detailed analysis and additional results.
Error bars show the binomial standard error of the mean (SEM), $\sqrt{p(1-p)/n}$, computed with bucket size $n$ (papers per score bin), where $p$ is the empirical best paper probability within each bin.
}
\label{fig:combined} 
\vspace{-0.1in}
\end{figure*}

Our experimental findings (Sections~\ref{sec:expe} and~\ref{exp:NeurIPS}) provide a potential explanation for the observed convexity in this setting. A key insight is that while a paper's true quality may be unbounded, the reviewer-assigned scores are inherently bounded. This reflects standard practice at major machine learning conferences like NeurIPS, ICML, and ICLR, where reviewer scores are restricted to predetermined ranges (e.g., [1, 10]). Thus, our score range reflects actual operational constraints rather than an artificially imposed limitation.
Under this bounded scoring regime, the ``best paper'' threshold is seldom reached in practice. Thus the high score in conferences does not effectively differentiate the ``best paper'', and thus, does not reach its conditional probability's saturation point. Consequently, it maintains convexity (Figure \ref{fig:second1}).

In contrast, when comparing accepted versus rejected papers, clear acceptance/rejection boundaries cause probability saturation, eliminating convexity (Figures~\ref{fig:probability_plot1} and~\ref{fig:probability_plot_nips2}). See Section~\ref{morecomment} for further discussion.

\begin{figure*}[ht]
  \centering
  \begin{subfigure}{.45\textwidth}
    \centering
    \includegraphics[width=1.2\textwidth]{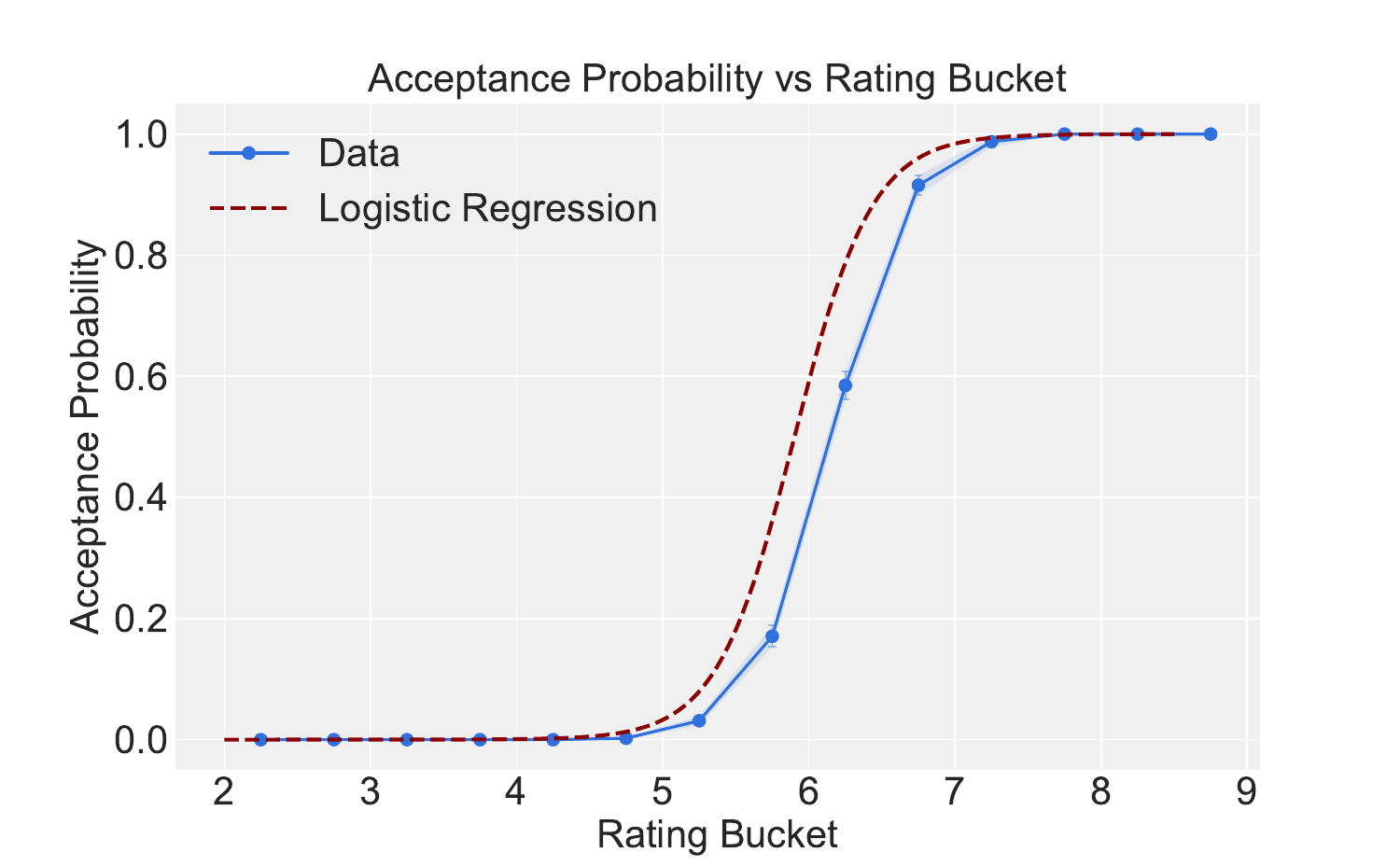}
    \caption{ Acceptance Probabilities at ICLR 2021}
    \label{fig:probability_plot1}
  \end{subfigure}
  \hfill  
  \begin{subfigure}{.45\textwidth}
    \centering
    \includegraphics[width=1.2\textwidth]{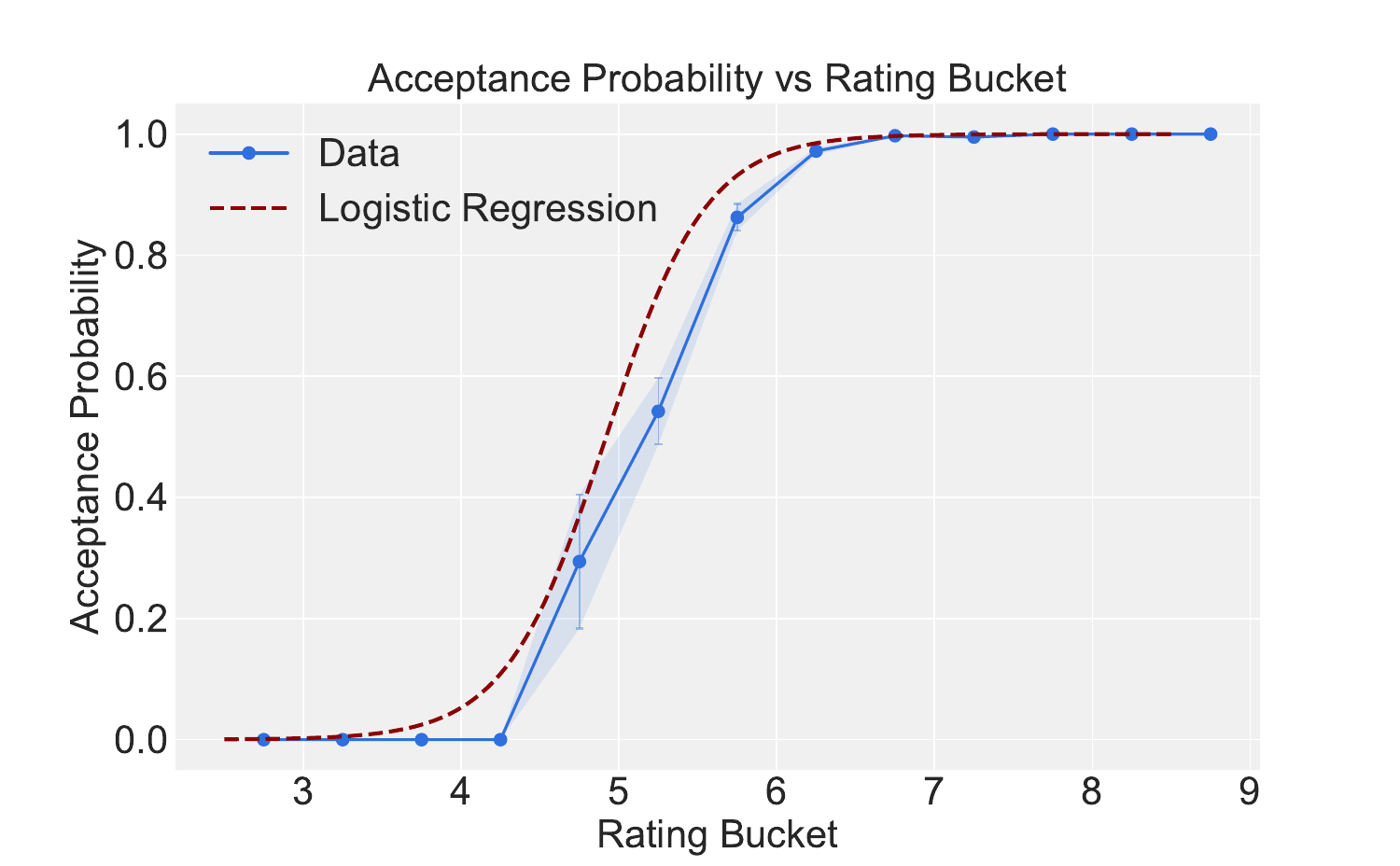}
    \caption{ Acceptance Probabilities at NeurIPS 2021}
    \label{fig:probability_plot_nips2}
  \end{subfigure} 
  \caption{
  Comparison of acceptance probabilities for ICLR~2021 and NeurIPS~2021.}
  \label{fig:combined_acceptance_probabilities}
\end{figure*}

Since previous conferences did not consider the provision of rankings, they actually correspond to the
\textbf{unlimited-quota} case, as demonstrated in Table \ref{tab:prac}. This implies that we have demonstrated the practicality of the Isotonic Mechanism in that case, as long as the author's goal is to maximize the expected probability of having at least one paper selected as the best paper, which is reasonable to be assumed.

\section{Theoretical Results}
\label{sec:main}
In this section, we develop the theoretical framework that underpins the Isotonic Mechanism for recommending best paper awards. First, we introduce the modeling, defining the key components of the peer-review environment and the Isotonic Mechanism. Second, we present the necessary technical background on the theories of majorization and Schur-convexity, which are central to our proofs. Finally, we state and prove our main theorems to show that the Isotonic Mechanism is truthful and individually rational under the relaxed convexity assumption, especially when the \emph{quota} is one.
\subsection{Modeling Framework}
\label{sec:modeling}
We begin by formalizing the components of our model, which captures the key interactions within the author-assisted peer-review process for best paper selection.

\paragraph{Knowledge Partition.}
The mechanism elicits information from authors in the form of a ranking. This information corresponds to a partition of the true score space $\mathbb{R}^n$ into a set of \emph{isotonic cones}, denoted by $\mathcal{S} = \{S_{\pi}\}$. Each cone $S_{\pi}$ corresponds to the author's assertion about the relative ordering of their papers' true quality scores:
\[
    S_{\pi} = \bigl\{\boldsymbol{x} \in \mathbb{R}^n \,\mid\, x_{\pi(1)} \geq x_{\pi(2)} \geq \cdots \geq x_{\pi(n)}\bigr\},
\]
where $\pi$ is a permutation of $\{1,2,\ldots,n\}$.

\paragraph{Authors.}
An author of $n$ submissions is assumed to privately know the vector of their papers' true, ground-truth quality scores, $\boldsymbol{R} = (R_1, R_2, \ldots, R_n)$. The author's objective is to act in a way that maximizes their expected utility, $\mathbb{E}[U(\boldsymbol{\hat{R}})]$, where the utility function $U$ depends on the final, adjusted scores $\boldsymbol{\hat{R}}$ of their papers.

\paragraph{Reviewers.}
The peer-review process introduces noise to the true scores. Each paper $i$ receives a noisy score $y_i$, resulting in an observed score vector $\boldsymbol{y} = \boldsymbol{R} + \boldsymbol{\epsilon}$, where $\boldsymbol{\epsilon}$ is a random noise vector.
\paragraph{Isotonic Mechanism.}
The author reports a ranking $\pi$, declaring that their true score vector $\boldsymbol{R}$ lies in the cone $S_{\pi}$. The mechanism then adjusts the noisy scores by solving an isotonic regression problem, which finds the closest point in the specified cone $S_{\pi}$ to the observed scores $\boldsymbol{y}$. Formally, it solves:
\begin{equation}\label{equ:general}
\begin{aligned}
    \min_{\boldsymbol{r}} \quad & \|\boldsymbol{y} - \boldsymbol{r}\|_{2}^{2} \\
    \text{s.t.} \quad & \boldsymbol{r} \,\in\, S_{\pi},
\end{aligned}
\end{equation}
where $\|\cdot\|_2$ is the Euclidean norm. The solution serves as the final adjusted score vector for the $n$ papers. This vector, denoted generally as  
$\boldsymbol{\hat{R}}$
  previously, is more formally written as  
$\boldsymbol{\hat{R}_{\pi}}$ 
  to emphasize its dependence on the ranking $\pi$.

 \paragraph{Best Paper Selection and Utility.}
The author's utility is derived from the selection process for best papers. We consider that each author has a \emph{quota} $k$, representing the number of their papers to be considered for the award. Their utility is a function of the adjusted scores of these top papers. We analyze two key scenarios which differ in the information available to the decision-maker:
\begin{enumerate}
  \item 
\textbf{Blind Case.} The decision-maker sees only the adjusted scores, not the author's ranking. The author's utility therefore depends symmetrically on the top $k$ scores, regardless of which paper earned them. The utility function is additive. For any possible score $\boldsymbol{x} \in \mathbb{R}^n$, it is defined as:
\[
    U(\boldsymbol{x}) 
    \;=\; U\bigl(x_{(1)}\bigr) + U\bigl(x_{(2)}\bigr) + \cdots + U\bigl(x_{(k)}\bigr),
\]
where $x_{(1)} \ge \cdots \ge x_{(n)}$ are the scores sorted in descending order. We assume $U$ is nondecreasing for $k=1$, and convex and nondecreasing for $k>1$.

\item \textbf{Informed Case.} The decision-maker knows the author's ranking and may use it. This introduces an asymmetry: the utility may depend on the specific rank a paper was given. For example, the author might value their self-declared best paper achieving a higher score than their second paper achieving the same score. This is modeled with potentially different utility functions for each rank:
\[
    U(\boldsymbol{x})
    \;=\; U_{1}\bigl(x_{(1)}\bigr) + U_{2}\bigl(x_{(2)}\bigr) + \cdots + U_{k}\bigl(x_{(k)}\bigr).
\]
Here, the functions $\{U_i\}_{i=1}^k$ are assumed to be convex and nondecreasing, with the additional structural assumption that $U_i' \ge U_{i+1}'$. For the special case of $k=1$, we only require $U_1$ to be nondecreasing.
\end{enumerate}

\paragraph{Model Assumptions.}
The utility structure is defined above in the \textbf{Blind} and \textbf{Informed} cases. Another key assumption concerns the review noise:

\begin{assumption}[ (\textbf{Exchangeable Noise})]
\label{Assumption3}
The additive noise $\boldsymbol{\epsilon}$ is \emph{exchangeable}. Formally, for any permutation $\pi$ of $\{1,\ldots,n\}$, the distribution of $(\epsilon_1, \ldots, \epsilon_n)$ is identical to that of $(\epsilon_{\pi(1)}, \ldots, \epsilon_{\pi(n)})$.
\end{assumption}

\noindent This assumption is a natural choice for modeling the peer review process. It captures the inherent symmetry from the perspective of an owner who must rank their items \emph{ex-ante}—that is, before reviewer assignments are made. At that stage, the owner has no basis to assume that the review noise for one item will be statistically different from another.

This condition is strictly weaker than assuming the noise is i.i.d. (independent and identically distributed). It allows for dependencies among the noise terms, such as a bias shared across all reviewers, as long as this effect is applied symmetrically to all items. A key advantage of this assumption is that it is non-parametric, requiring no specific modeling of the noise's distribution family. For more discussions on this assumption, see Section 6 in \citep{wuIsotonicMechanismOverlapping2023}.
 
\subsection{Preliminaries and Technical Backgrounds}
\begin{definition}
  [ \citep{marshallInequalitiesTheoryMajorization2011a}] We say that a vector $\boldsymbol{a} \in \mathbb{R}^n$ majorizes another vector $\boldsymbol{b} \in \mathbb{R}^n$, denoted by $\boldsymbol{a} \succeq \boldsymbol{b}$, if \begin{equation}
\sum_{i=1}^{k} a_{(i)} \geq \sum_{i=1}^{k} b_{(i)}
\end{equation}
with $k =1,2,\ldots,n$, and $\sum _{i=1}^na_i = \sum_{i=1}^n b_i$, where $a_{(1)}\ge a_{(2)}\ge \cdots \ge a_{(n)}$ and $b_{(1)}\ge b_{(2)}\ge \cdots \ge b_{(n)}$ are sorted in descending order from $\boldsymbol{a}$ and $\boldsymbol{b}$. If we drop the requirement of  $\sum _{i=1}^na_i = \sum_{i=1}^n b_i$, then $\boldsymbol{a}$ is called to weakly majorize $\boldsymbol{b}$, denoted by $\boldsymbol{a} \succeq_{\mathrm{w}} \boldsymbol{b}$.
\end{definition} 
\begin{definition}
  [ \citep{suTruthfulOwnerAssistedScoring2022a}] We say that a vector $\boldsymbol{a} \in \mathbb{R}^n$ majorizes another vector $\boldsymbol{b} \in\mathbb{R}^n$ in the natural order, denoted by $\boldsymbol{a}\succeq_{\mathrm{no}}\boldsymbol{b}$, if \begin{equation}
\sum_{i=1}^{k} a_{i} \geq \sum_{i=1}^{k} b_{i}
\end{equation} for $k=1,2,\ldots,n$ and $\sum _{i=1}^na_i = \sum_{i=1}^n b_i$.
\end{definition}
Now, define the solution to the optimization program\begin{equation}\label{equ:iso}
\begin{array}{cl}\min _{\boldsymbol{r}} & \|\boldsymbol{y}-\boldsymbol{r}\|^{2} \\ \text { s.t. } & r_1\ge r_2 \ge \cdots \ge r_n,\end{array}
\end{equation}
to be $\boldsymbol{y^+}$, which is just the projection of $\boldsymbol{y}$ onto the isotonic cone associated with the identity permutation.
\begin{lemma}[ \citep{suTruthfulOwnerAssistedScoring2022a}]
\label{lem:no}
$ \text { If } \boldsymbol{a} \succeq_{\mathrm{no}} \boldsymbol{b} \text {, then we have } \boldsymbol{a}^{+} \succeq\boldsymbol{b}^{+} \text {. }$
\end{lemma}
\begin{definition}
  [ \citep{marshallInequalitiesTheoryMajorization2011a}] A real-valued function $\phi$ defined on $\mathbb{R}^n$ is said to be Schur-convex if $\boldsymbol{x} \succeq \boldsymbol{y} \Rightarrow \phi(x) \ge \phi(y).$
\end{definition}
\begin{lemma}[ \citep{dykstraMajorizationLorenzOrder1988a}] Let $\boldsymbol{a}, \boldsymbol{b}$ be two vectors in $\mathbb{R}^n$.
	\begin{itemize}
		\item The inequality \begin{equation}
			\sum_{i=1}^nh(a_i) \ge \sum_{i=1}^nh(b_i),
		\end{equation} holds for all nondecreasing convex functions $h$ if and only if $\boldsymbol{a} \succeq_{\mathrm{w}} \boldsymbol{b}$.
		\item The inequality 
\begin{equation}
			\sum_{i=1}^nh(a_i) \ge \sum_{i=1}^nh(b_i),
		\end{equation} holds for all convex functions $h$ if and only if $\boldsymbol{a} \succeq \boldsymbol{b}$.
	\end{itemize}
	
\end{lemma}
\begin{definition}
	A linear transformation is called
a $T$-transformation, or more briefly a $T$-transform, if there is a number $0 \leq \lambda \leq 1$ and a permutation matrix $Q$ that only interchanges two coordinates, such that \begin{equation}
	T = \lambda I + (1-\lambda)Q,
\end{equation}Thus $\boldsymbol{x}T$ has the form:
\begin{equation}
\begin{aligned} \boldsymbol{x} T= & \left(x_{1}, \ldots, x_{j-1}, \lambda x_{j}+(1-\lambda) x_{k}, x_{j+1}, \ldots, x_{k-1}\lambda x_{k}+(1-\lambda) x_{j}, x_{k+1}, \ldots, x_{n}\right).\end{aligned}
\end{equation}
\begin{lemma}[ \citep{muirheadMethodsApplicableIdentities1902a}]
	\label{lem:ttrans}
	 If $\boldsymbol{a} \succeq \boldsymbol{b}$, then $\boldsymbol{b}$ can be derived from $\boldsymbol{a}$ by successive applications of finite number of $T$-transforms.
\end{lemma}
\end{definition}
In proving that a function $\phi$ is Schur-convex, it is often helpful
to realize that, in effect, one can take $n=2$ without loss of generality.
This fact is a consequence of this lemma, which says that if $\boldsymbol{a} \succeq \boldsymbol{b}$, then $\boldsymbol{b}$ can
be derived from $\boldsymbol{a}$ by a finite number of $T$-transforms. Consequently,
it is sufficient to prove that $\phi(\boldsymbol{a}) \ge \phi(\boldsymbol{b})$ when $\boldsymbol{a} \succeq \boldsymbol{b}$ and $\boldsymbol{a}$ differs
from $\boldsymbol{b}$ in only two components, so that all but two arguments of $\phi$ are
fixed.
\begin{lemma}[ \citep{boydConvexOptimization2004a}] 
\label{lem:convexorder}
	$f_k(\boldsymbol{x}) = x_{(1)}+x_{(2)}+\cdots+x_{(k)}$ is convex and symmetric with respect to $\boldsymbol{x} \in \mathbb{R}^n,$ for fixed $k$.
\end{lemma}

\subsection{Main Theoretical Guarantees}\label{app:proof}
 
With the necessary theoretical tools in place, we now present the main theorems that establish the incentive compatibility (truthfulness) and individual rationality of our mechanism. Our proof strategy hinges on a key criterion connecting Schur-convexity to truthfulness.

\begin{lemma}[ (Schur-Convexity Criterion for Truthfulness)]
\label{lem:schur_truthful}
Under the assumption of exchangeable noise (Assumption~\ref{Assumption3}) and the modeling framework established in Section~\ref{sec:modeling}, if an author's utility function $U(\boldsymbol{x})$ is Schur-convex, then the Isotonic Mechanism is truthful. That is, the author's expected utility is maximized by reporting their true ranking.
\end{lemma} 
\begin{proof}
Without loss of generality, assume the true ranking of the scores is the identity permutation, $\pi^* = \mathrm{id}$, such that the true score vector $\boldsymbol{R}$ satisfies $R_1 \ge R_2 \ge \cdots \ge R_n$. If an author reports a different ranking $\pi \neq \pi^*$, the Isotonic Mechanism solves the following optimization problem:
\begin{equation}\label{equ:iso}
\begin{array}{cl}\min _{\boldsymbol{r}} & \|\boldsymbol{y}-\boldsymbol{r}\|^{2} \\ \text { s.t. } & r_{\pi(1)}\ge r_{\pi(2)} \ge \cdots \ge r_{\pi(n)},\end{array}
\end{equation}
Because the squared Euclidean norm is invariant to permutations, this problem is equivalent to finding a vector $\boldsymbol{r}' = \pi \circ \boldsymbol{r}$ that solves:
\begin{equation}\label{equ:iso}
\begin{array}{cl}\min _{\boldsymbol{r}'} & \|\pi\circ\boldsymbol{y}-\boldsymbol{r}'\|^{2} \\ \text { s.t. } & r'_{1}\ge r'_{2} \ge \cdots \ge r'_{n}.\end{array}
\end{equation}
The solution to this standard isotonic regression problem is $\boldsymbol{r}' = (\pi \circ \boldsymbol{y})^+$. Transforming back, the final adjusted score vector is $\boldsymbol{\hat{R}}_\pi = \pi^{-1} \circ (\pi \circ \boldsymbol{y})^+$.

We now compare the expected utility of reporting $\pi$ versus the true ranking $\pi^*$. The expected utility for reporting $\pi$ is $\mathbb{E}[U(\boldsymbol{\hat{R}}_\pi)] = \mathbb{E}[U(\pi^{-1}\circ (\pi \circ \boldsymbol{y})^+)]$. A key property of a Schur-convex function is that it is \emph{symmetric}, meaning $U(\pi^{-1} \circ \boldsymbol{z}) = U(\boldsymbol{z})$ for any vector $\boldsymbol{z}$. Therefore, $U(\boldsymbol{\hat{R}}_\pi) = U((\pi \circ \boldsymbol{y})^+)$. Substituting $\boldsymbol{y} = \boldsymbol{R} + \boldsymbol{\epsilon}$, we get:
\begin{equation}
    \mathbb{E}[U(\boldsymbol{\hat{R}}_\pi)] = \mathbb{E}[U( (\pi \circ \boldsymbol{R} + \pi \circ \boldsymbol{\epsilon})^+ )].
\end{equation}
Under Assumption~\ref{Assumption3}, the noise vector $\boldsymbol{\epsilon}$ is \emph{exchangeable}, so $\pi \circ \boldsymbol{\epsilon}$ has the same distribution as $\boldsymbol{\epsilon}$. Thus, the expectation remains unchanged if we replace $\pi \circ \boldsymbol{\epsilon}$ with $\boldsymbol{\epsilon}$:
\begin{equation} \label{eq:utility_pi}
    \mathbb{E}[U(\boldsymbol{\hat{R}}_\pi)] = \mathbb{E}[U( (\pi \circ \boldsymbol{R} + \boldsymbol{\epsilon})^+ )].
\end{equation}
For the true ranking $\pi^*=\mathrm{id}$, the adjusted scores are simply $\boldsymbol{\hat{R}}_{\pi^*} = \boldsymbol{y}^+ = (\boldsymbol{R}+\boldsymbol{\epsilon})^+$.

By definition, a vector sorted in descending order majorizes any permutation of itself in the natural order ($\succeq_{\mathrm{no}}$). Thus, $\boldsymbol{R} \succeq_{\mathrm{no}} \pi \circ \boldsymbol{R}$, which implies $\boldsymbol{R}+\boldsymbol{\epsilon} \succeq_{\mathrm{no}} \pi \circ \boldsymbol{R}+\boldsymbol{\epsilon}$. From Lemma~\ref{lem:no}, it follows that the projections also satisfy the majorization relation:
\begin{equation}
    (\boldsymbol{R}+\boldsymbol{\epsilon})^+ \succeq (\pi \circ \boldsymbol{R}+\boldsymbol{\epsilon})^+.
\end{equation}
Since $U$ is Schur-convex, this vector majorization implies an inequality:
\begin{equation}
    U( (\boldsymbol{R}+\boldsymbol{\epsilon})^+) \ge U( (\pi \circ \boldsymbol{R}+\boldsymbol{\epsilon})^+).
\end{equation}
This inequality holds for any realization of the noise $\boldsymbol{\epsilon}$. Taking the expectation over $\boldsymbol{\epsilon}$ on both sides yields:
\begin{equation}
    \mathbb{E}[U( (\boldsymbol{R}+\boldsymbol{\epsilon})^+)] \ge \mathbb{E}[U( (\pi \circ \boldsymbol{R}+\boldsymbol{\epsilon})^+)].
\end{equation}
This is equivalent to $\mathbb{E}[U(\boldsymbol{\hat{R}}_{\pi^*})] \ge \mathbb{E}[U(\boldsymbol{\hat{R}}_{\pi})]$. This shows that the expected utility is maximized by reporting the true ranking $\pi^*$. Thus, the Isotonic Mechanism is truthful.
\end{proof}
Lemma~\ref{lem:schur_truthful} provides a clear path forward: to prove our mechanism is truthful for the utility models defined, we need only show that they are indeed Schur-convex.
\begin{theorem}\label{thm:betteru}
Suppose the true score vector $\boldsymbol{R}$ lies in $S_{\pi^*} \in \mathcal{S}$. The author's expected utility is maximized when the Isotonic Mechanism is provided with the ground-truth ranking $\pi^*.$
Specifically, when the {quota} is one, the Isotonic Mechanism remains truthful with nondecreasing but not necessarily convex utility functions.
\end{theorem}
 
\begin{proof}
Our strategy is to prove that the author's utility function is Schur-convex under the utility structures specified for the \textbf{Blind} and \textbf{Informed Cases} in Section \ref{sec:modeling}. The truthfulness of the mechanism then follows directly from the Schur-Convexity Criterion by  Lemma~\ref{lem:schur_truthful}. Both utility functions are symmetric by definition, as they depend only on the sorted values of the scores.

\paragraph{Case 1: Blind Case.}  Note that function $f_k(\boldsymbol{x})=x_{(1)}+x_{(2)}+\cdots+x_{(k)}$ is symmetric and convex by Lemma \ref{lem:convexorder}. We still utilize properties of $T$-transform. For any vector $\boldsymbol{x}$ in the standard isotonic cone, consider $\boldsymbol{x}T$ where $T$ is a $T$-transform and $\boldsymbol{x}T$ stays at the standard isotonic cone. Then 
	\begin{equation}
\begin{aligned} \boldsymbol{x} T= & \left(x_{1}, \ldots, x_{j-1}, \lambda x_{j}+(1-\lambda) x_{k}, x_{j+1}, \ldots, x_{k-1}\lambda x_{k}+(1-\lambda) x_{j}, x_{k+1}, \ldots, x_{n}\right).\end{aligned}
\end{equation}
and $f_k(\boldsymbol{x}T) \leq \lambda f_k(\boldsymbol{x})+(1-\lambda)f_k(\tilde{\boldsymbol{x}})$
by convexity. Here $\tilde{\boldsymbol{x}} = (x_1,\ldots,x_k,\ldots,x_j,\ldots,x_n)$, $j \leq k$. By symmetry, $f_k(\boldsymbol{x}) = f_k(\tilde{\boldsymbol{x}})$. We conclude that $f_k(\boldsymbol{x})$ is Schur-convex itself. By applying a nondecreasing function $U$, the result is immediate.	 
\paragraph{Case 2: Informed Case.} First, we can add a 0 utility: 0($x_{(g)}$) = 0 for $g \in \{k+1,\ldots,n\}$. Note that $U_1(x_{(1)})+U_2(x_{(2)})+\cdots + U_k(x_{(k)})$ is symmetric, so we can restrict the domain to be the standard isotonic cone \begin{equation}
	\left\{\boldsymbol{x}: x_{1} \geq x_{2} \geq \cdots \geq x_{n}\right\}.
\end{equation} 
For two vectors $\boldsymbol{x}, \boldsymbol{y}$ in this cone, suppose $\boldsymbol{x} \succeq \boldsymbol{y}$, then we can apply a finite number of $T$-transforms to $\boldsymbol{x}$ to get $\boldsymbol{y}$; that is, $\boldsymbol{y} = \boldsymbol{x}T_1T_2\cdots T_l$, where $\boldsymbol{x}T_q$ has the form:
\begin{equation}
\begin{aligned} \boldsymbol{x} T_q= & \left(x_{1}, \ldots, x_{j-1}, \lambda_q x_{j}+(1-\lambda_q) x_{k}, x_{j+1}, \ldots, x_{k-1}\lambda_q x_{k}+(1-\lambda_q) x_{j}, x_{k+1}, \ldots, x_{n}\right),\end{aligned} 
\end{equation} 
with $0 \leq \lambda_q \leq 1.$ Moreover, in the original proof of Lemma \ref{lem:ttrans}, every $T$-transform $T_q,$ $1\leq q \leq l$ will preserve the ranking of the vector $\boldsymbol{x}$. That is, $\boldsymbol{x}T_q$ is still in the standard isotonic cone.

Thus, given $x_1 \ge x_2 \ge \cdots\ge  x_j \ge \cdots \ge x_k\ge  \cdots \ge x_n$, after applying a ranking-preserving $T$-transform, we get a new vector with coordinates $x_1 \ge x_2 \ge \cdots\ge  \alpha x_j + (1-\alpha)x_k \ge \cdots \ge \alpha x_k+(1-\alpha)x_j\ge  \cdots \ge x_n$. To prove $U_1(x_{(1)})+U_2(x_{(2)})+\cdots + U_k(x_{(k)})$  is Schur-convex, it suffices to prove the inequality
	 \begin{equation}
		U_j(\alpha x_j+(1-\alpha)x_k)+U_k((1-\alpha) x_j+\alpha x_k) \leq U_j(x_j)+U_k(x_k).
	\end{equation}

	Since $U_j,U_k$ are convex, it suffices to show\begin{equation}
		U_j(x_k)+U_k(x_j)\leq U_j(x_j)+U_k(x_k),
	\end{equation}
	which is obvious then by the assumption $U_j' \ge U_k'$ and Cauchy's mean-value theorem.
\end{proof}
With Theorem \ref{thm:betteru}, obtaining Table \ref{tab:prac} is indeed just one step away:
\begin{itemize}
	\item \textbf{Blind Case}: the utility function takes the form $U(x_{(1)})+U(x_{(2)})+\cdots + U(x_{(k)})$ which is a special case of Theorem \ref{thm:betteru}; $k=1$ is satisfied by any nondecreasing $U$; when $k\ge 2$, we require $U$ to be convex.
	\item \textbf{Informed Case}: $k=1$ is the same case as the \textbf{Blind Case}. When $k \ge 2$, apply Theorem \ref{thm:betteru} to prove the truthfulness.
\end{itemize}

Theorem~\ref{thm:betteru} establishes the crucial property of truthfulness for our mechanism, ensuring that a rational author's optimal strategy is to provide their true ranking of their papers. However, incentive compatibility alone is not sufficient to guarantee voluntary participation. A critical follow-up question remains: does an author actually benefit from this mechanism, or could the process of score adjustment potentially harm their expected outcome compared to the current practice of using the original noisy scores? For instance, an author might fear that even with a truthful ranking, the resulting adjusted score~$\hat{\boldsymbol{R}}$ could, in expectation, yield a lower utility than the raw noisy score~$\boldsymbol{y}$. To secure participation, we must therefore show that the mechanism is not only truthful but also individually rational and beneficial. Theorem \ref{prop:k1} addresses this very concern by demonstrating that, for the important case of a single nomination (\emph{quota} of one), an author's expected utility is indeed improved (or at least not diminished) by the Isotonic Mechanism.
\begin{theorem}[ (Individual Rationality for Quota One)]
\label{prop:k1}
When the \emph{quota} is one, 
the author's expected utility from the Isotonic Mechanism is no less than the expected utility derived from the noisy score alone, i.e.:
\[
    \mathbb{E}\bigl[U\bigl(\hat{R}_{\pi^*(1)}\bigr)\bigr] 
    \;\ge\; 
    \mathbb{E}\bigl[U\bigl(y_{\pi^*(1)}\bigr)\bigr],
\]
assuming $U$ is only nondecreasing.
\end{theorem}
 \begin{proof}
Without loss of generality, let the author's true ranking be the identity permutation, $\pi^* = \mathrm{id}$. The adjusted score for the top-ranked paper is then the first component of the vector $\boldsymbol{\hat{R}} = \boldsymbol{y}^+$, which is the Euclidean projection of the noisy score vector $\boldsymbol{y}$ onto the standard isotonic cone, $K = \{\boldsymbol{x} \in \mathbb{R}^n \mid x_1 \ge x_2 \ge \cdots \ge x_n\}$. The theorem's claim is thus $\mathbb{E}[U(y^+_1)] \ge \mathbb{E}[U(y_1)]$.

Since the function $U$ is nondecreasing, this inequality is a direct consequence of showing that $y_1 \le y^+_1$ for any vector $\boldsymbol{y} \in \mathbb{R}^n$. We prove this property below.

A fundamental property of a projection onto a convex set $K$ states that for any point $\boldsymbol{w} \in K$:
\begin{equation} \label{eq:proj_property} 
    (\boldsymbol{y} - \boldsymbol{y}^+)^\sT (\boldsymbol{w} - \boldsymbol{y}^+) \le 0.
\end{equation}
First, we establish that the error vector $(\boldsymbol{y} - \boldsymbol{y}^+)$ is orthogonal to the projection $\boldsymbol{y}^+$. Since $K$ is a cone, if $\boldsymbol{y}^+ \in K$, then both $\boldsymbol{0}$ and $2\boldsymbol{y}^+$ are also in $K$.  
\begin{itemize}
    \item Let $\boldsymbol{w} = \boldsymbol{0}$. Equation~\eqref{eq:proj_property} gives $(\boldsymbol{y} - \boldsymbol{y}^+)^\sT (-\boldsymbol{y}^+) \le 0$, which implies $(\boldsymbol{y} - \boldsymbol{y}^+)^\sT \boldsymbol{y}^+ \ge 0$.   
    \item Let $\boldsymbol{w} = 2\boldsymbol{y}^+$. Equation~\eqref{eq:proj_property} gives $(\boldsymbol{y} - \boldsymbol{y}^+)^\sT (\boldsymbol{y}^+) \le 0$.
\end{itemize}
Combining these two inequalities, we conclude that $(\boldsymbol{y} - \boldsymbol{y}^+)^\sT \boldsymbol{y}^+ = 0$.  

Substituting this orthogonality result back into the projection property in Equation~\eqref{eq:proj_property}, we obtain a simpler inequality:
\begin{equation}
    (\boldsymbol{y} - \boldsymbol{y}^+)^\sT \boldsymbol{w} \le 0 \quad \text{for all } \boldsymbol{w} \in K.
\end{equation}
Now, consider the specific vector $\boldsymbol{w}^* = (1, 0, \dots, 0)$. This vector belongs to the standard isotonic cone $K$ because its components are non-increasing. Setting $\boldsymbol{w} = \boldsymbol{w}^*$ yields:
\begin{equation}
    (\boldsymbol{y} - \boldsymbol{y}^+)^\sT \boldsymbol{w}^* = y_1 - y^+_1 \le 0.
\end{equation}
This directly implies $y_1 \le y^+_1$.

Since $y_1 \le y^+_1$ and $U$ is nondecreasing, it follows that $U(y_1) \le U(y^+_1)$. This holds for any realization of the noisy scores $\boldsymbol{y}$. Taking the expectation over the distribution of $\boldsymbol{y}$ preserves the inequality, which completes the proof.
\end{proof}

\section{Data Analysis of ICLR 2019---2023}\label{sec:expe}
 
In this section, we use publicly accessible review data from ICLR 2019---2023 \citep{ICLR2019b, ICLR2020, ICLR2021, ICLR2022, bertoCrawlVisualizeICLR2023a} to examine whether the utility function can be regarded as convex in practice. We interpret utility in terms of conditional selection probabilities,
\begin{equation}\label{equ:natu2}
    U(\boldsymbol{\hat{R}})\;\sim\;\sum_{i=1}^{k}\,
   \mathbb{P}\;\!\bigl(\text{Chosen as best papers}\;\big|\;\hat R_{(i)}\bigr),
\end{equation}
where the \emph{quota} $k$ is the maximum number of top papers an author may nominate.  
Because authors’ rankings are unavailable in the public data, we necessarily treat every accepted paper as if its authors had nominated it, i.e.\ we analyse the \textbf{unlimited-quota} regime, where the nomination budget satisfies $k$ exceeding the maximum number of papers by any individual author.   

All accepted papers are therefore pooled to fit a logistic model for $\mathbb{P}(\text{best} \mid \hat R)$, yielding empirical evidence that $U$ is convex and non-decreasing under this regime. This empirical confirmation underpins the \checkmark\ entry at the \textbf{unlimited-quota} column for the \textbf{Blind Case} in Table~\ref{tab:prac}. Further implementation details appear in Appendices~\ref{app:expe}.


%
%



\begin{figure*}[t]
    \centering
    \begin{subfigure}{.45\textwidth}
        \centering
        \includegraphics[width=1.2\textwidth]{figures/ICLR2021accept.pdf}
        \caption{ Acceptance Probabilities at ICLR 2021}
        \label{fig:probability_plot}
    \end{subfigure}
    \hfill
    \begin{subfigure}{.45\textwidth}
        \centering
        \includegraphics[width=1.2\textwidth]{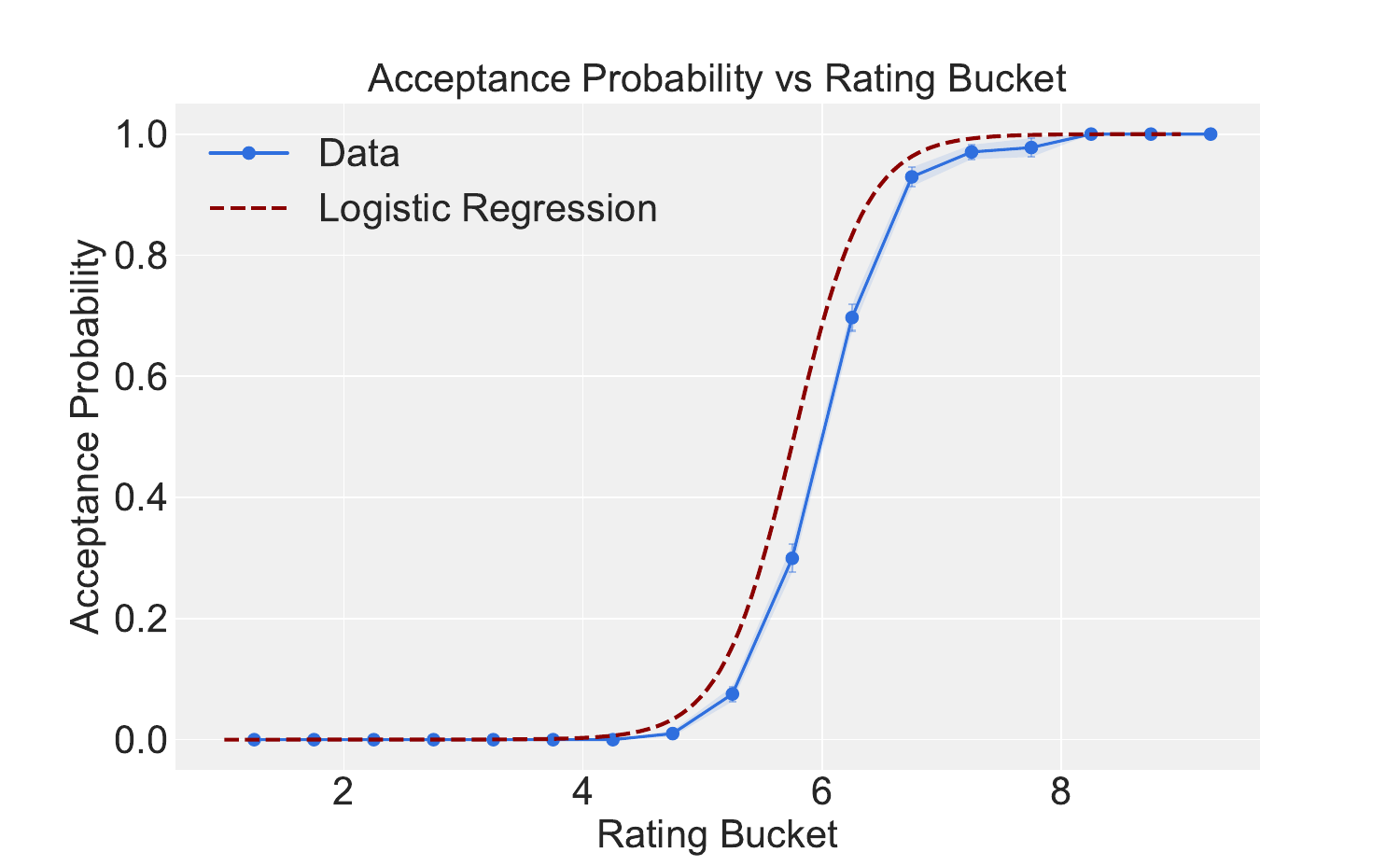}
        \caption{ Acceptance Probabilities at ICLR 2022}
        \label{fig:probability_plot2022}
    \end{subfigure}
    \caption{Comparison of Acceptance Probabilities at ICLR 2021 and ICLR 2022.}
    \label{fig:combined_acceptance_probabilitiesiclr}
\end{figure*}

\subsection{Data Analysis of Conditional Probability}
To empirically test our assumption about the convexity of the utility function, we first analyze the conditional probability of several related outcomes. Our analysis begins with the most fundamental decision in peer review: paper acceptance versus rejection. By establishing the shape of this probability curve, we can create a crucial baseline for comparison against the ``best paper''  selection scenario.

We start by illustrating our approach using the accept/reject decisions from ICLR 2021 and 2022, shown in Figure~\ref{fig:combined_acceptance_probabilitiesiclr}.

\paragraph{Analytical Approach.}
Our goal is to model the conditional probability $\mathbb{P}(\text{Accepted} \mid \text{Score})$. The independent variable is the average of all reviewer scores for a given paper. The dependent variable is the binary outcome (1 for accepted, 0 for rejected). We fit this relationship using a logistic regression model, a standard choice for modeling probabilities. The dashed orange line in the figures represents the fitted curve from this model. To handle the continuous nature of the scores and avoid instability from data sparsity at any single score-point, we group scores into discrete buckets. The jagged blue line shows the raw acceptance ratio within each bucket. This approach provides a stable, empirical basis for fitting our model. For the ICLR 2022 data, this logistic model proves to be a strong fit, achieving 90\% validation accuracy.

Across Figures~\ref{fig:combined}---\ref{fig:combined_spotlight_probabilities}, the plotted error bars report the binomial standard error of the mean (SEM), $\sqrt{p(1-p)/n}$, where $p$ is the empirical rate in a bucket and $n$ is the bucket size (number of papers) for that score range.

\paragraph{Results and Interpretation: The Saturation Effect.}
The primary takeaway from Figure~\ref{fig:combined_acceptance_probabilitiesiclr} is that the probability of acceptance follows a clear \textbf{sigmoidal (S-shaped) curve}. This shape reveals a critical ``saturation effect'' for high scores. 
At low scores, the probability of acceptance is near zero. As the score increases, the probability rises sharply, particularly in the middle range of scores (approximately 4.5 to 6.5), which can be interpreted as the region of greatest uncertainty for the acceptance decision.
Crucially, once a paper's score surpasses a certain threshold (e.g., a score of 7), the curve flattens and approaches 1.0. This saturation implies that further increases in the score yield diminishing returns for the probability of acceptance. A paper with a score of 8 is not much more likely to be accepted than a paper with a score of 7, because both are already safely above the acceptance bar. This non-convex behavior for high scores is a key characteristic of the standard acceptance process. As we will demonstrate, this stands in stark contrast to the best paper selection scenario.
\begin{figure}[htbp]
    \centering
    \includegraphics[width=0.8\textwidth]{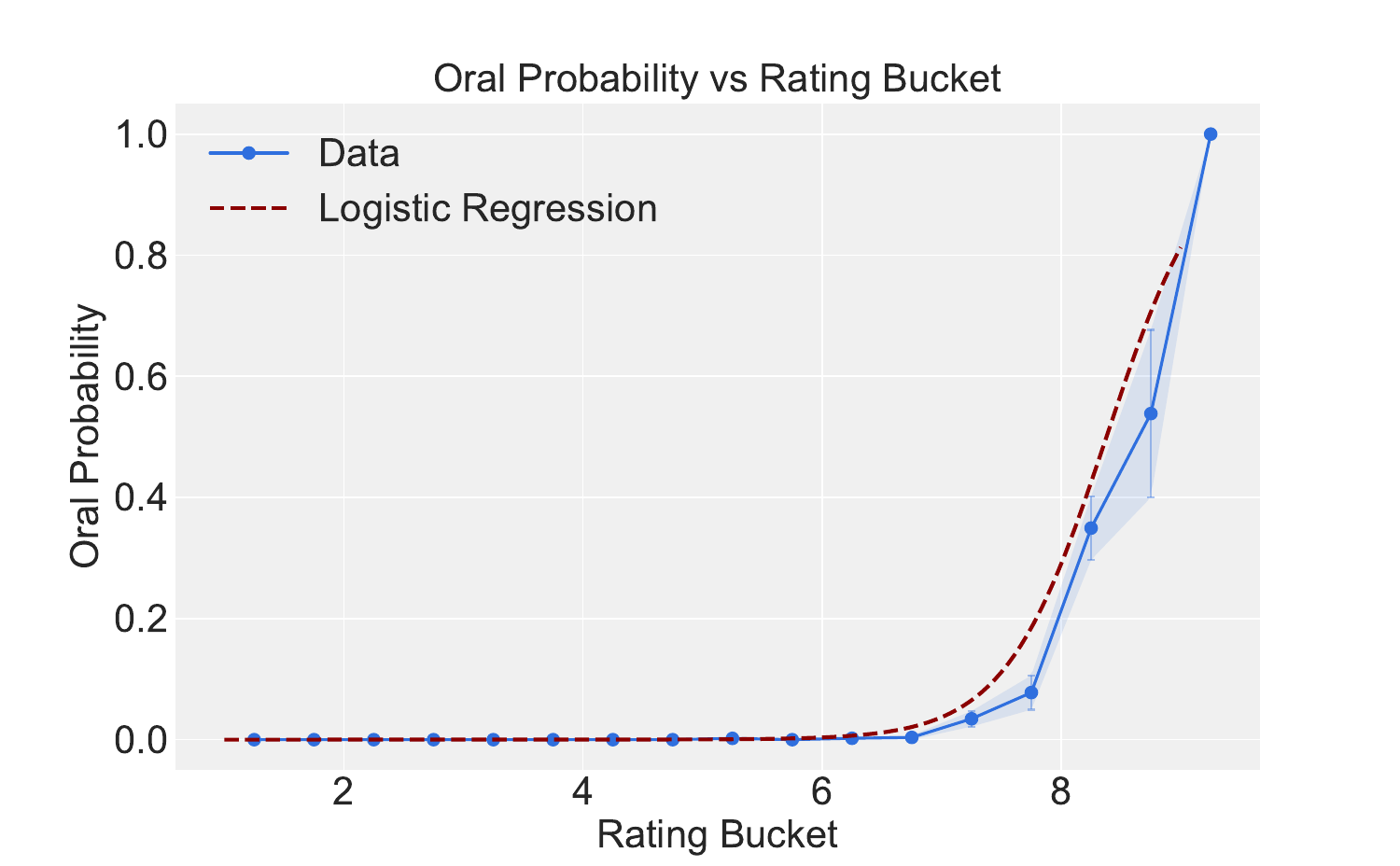}
    \caption{Oral presentation probability at ICLR 2022. The saturation point is visibly shifted to the right, occurring at much higher scores compared to the general acceptance curve. This illustrates how a more stringent selection criterion alters the probability distribution.}
    \label{fig:oral_plot1}
\end{figure}

 \label{morecomment}
Our analysis also reveals that the shape of the conditional probability curve changes systematically as we move up the hierarchy of paper outcomes. We posit that as the selection criterion becomes more stringent, the ``decision boundary'' shifts to higher scores, and the saturation effect that characterizes the general acceptance curve diminishes, eventually giving way to pure convexity for the top tier.

\paragraph{Intermediate Tiers: The Shifting Saturation Point.}
We first validate this hypothesis on an intermediate tier: Oral presentations. As shown in Figure~\ref{fig:oral_plot1} for ICLR 2022, the probability curve for being selected for an Oral presentation is still sigmoidal, but its saturation point has shifted significantly to the right compared to the general acceptance curve (Figure~\ref{fig:probability_plot2022}). The region of rapid probability increase now occurs at much higher scores (roughly 7.0---9.0). This indicates that the threshold for being a ``strong candidate'' for an Oral presentation is much higher. While saturation still occurs, it happens at the very top of the scoring range. This observation acts as a crucial bridge in our argument: if a higher bar for selection pushes the saturation point to the right, what happens when the bar is set at the absolute highest level, for the best paper award?

\paragraph{Methodology for Best Paper Analysis.}
Analyzing the probability of receiving a best paper award presents a statistical challenge: such awards are rare events, leading to sparse data in any single conference year. To create a robust dataset, we therefore combine data from all years of   ICLR data. This requires normalizing the various scoring systems to a consistent 0--9 scale (Figure~\ref{fig:bestlog}). Furthermore, we employ a non-uniform bucketing strategy, using finer divisions in the high-score regions where distinctions are most critical, and coarser divisions at the low end where best paper candidates are virtually non-existent.

\paragraph{Evidence for Convexity in Best Paper Selection.}
The analysis of this combined dataset reveals a starkly different pattern. As hypothesized, the probability curve for best paper awards does not saturate within the observable score range. Instead, it exhibits clear convexity, rising at an accelerating rate as scores increase. The formal test for this is the second derivative of the fitted logistic regression curve  (Figure~\ref{fig:second1}). A consistently positive second derivative confirms convexity. As our results show, this condition holds, providing strong empirical evidence that the utility function, when interpreted as the probability of winning a best paper award, is indeed convex.

This finding can be explained by   ``unbounded quality''. While review scores are capped at 10, the true quality of a groundbreaking paper is, for all practical purposes, unbounded. The 0--10 scale is insufficient to capture the vast differences between ``excellent'' papers and truly ``outstanding,'' field-defining work. The convex curve we observe is likely the initial, explosive growth phase of a probability distribution whose saturation point lies far beyond the artificial boundary of the scoring system.
 
\section{Data Analysis of NeurIPS 2021---2023}
\label{exp:NeurIPS} 
We conducted a similar analysis on NeurIPS 2021---2023 \citep{PaperCopilotNeurIPS,  PaperCopilotNeurIPSa, PaperCopilotNeurIPSb}. It can be observed that the results presented in the data are nearly identical with ICLR. This provides stronger evidence for our claims.

Figure \ref{fig:probability_plot_nips1} shows that the acceptance probability exhibits the same shape as ICLR (see Figure \ref{fig:probability_plot}) and reaches its ``saturation point'' around score 5. 
\begin{figure}[htbp]
    \centering
    \includegraphics[width=0.8\textwidth]{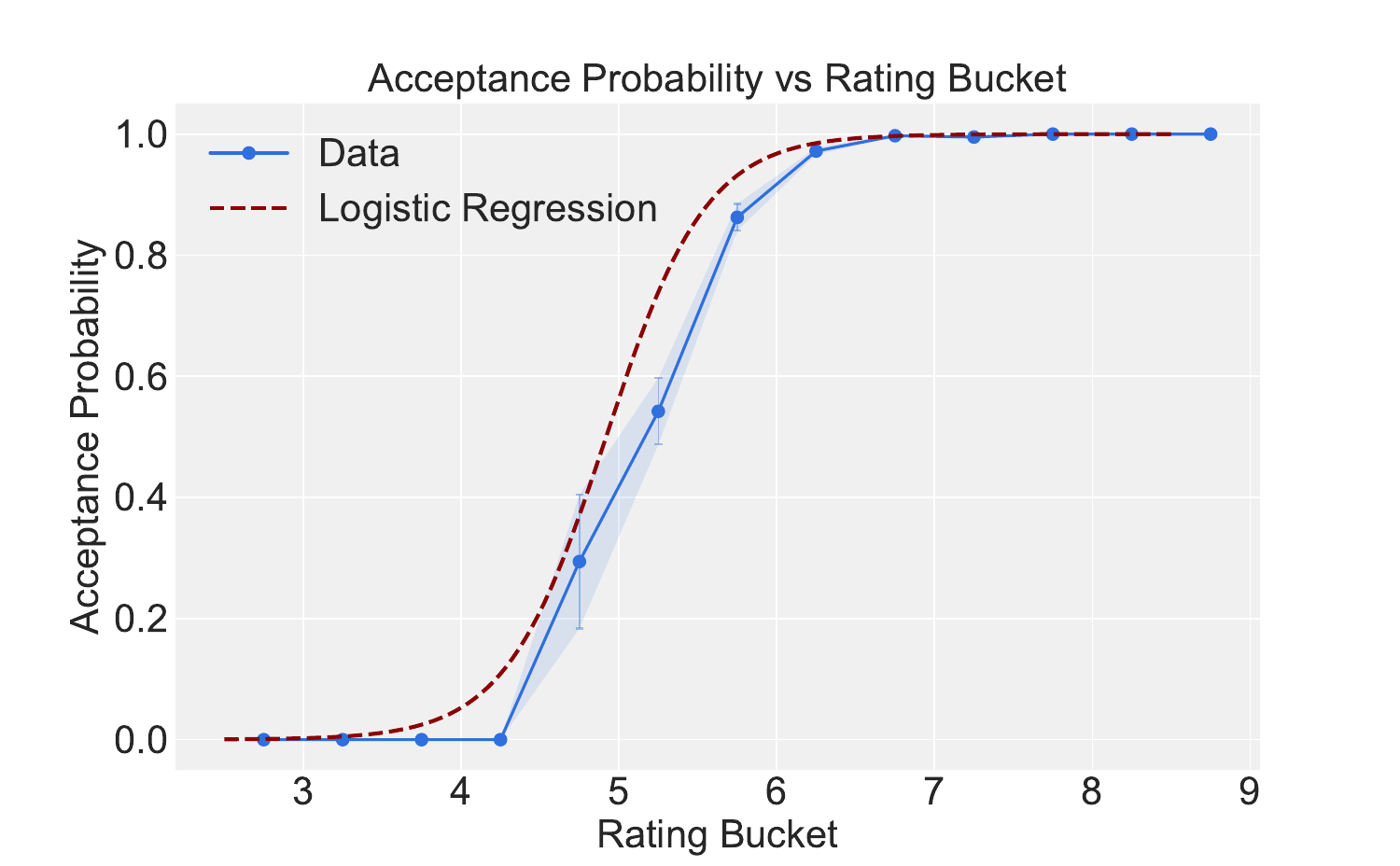}
    \caption{Acceptance probability at NeurIPS 2021. Similar to the ICLR data, this curve exhibits a clear saturation effect for high scores, demonstrating non-convexity in the standard acceptance regime.}
    \label{fig:probability_plot_nips1}
\end{figure}
Figures \ref{fig:oral_plot2}, \ref{fig:oral_plot3} show that there is  a significant shift in the saturation point for the conditional probability of the Spotlight (Highlighted) papers, and this curve has a longer interval on which the conditional probability exhibits convexity. It is of the same analysis of  Figure \ref{fig:oral_plot1}.
\begin{figure}[h]
  \centering

  \begin{subfigure}[b]{.45\textwidth}
    \centering
    \includegraphics[width=1.2\linewidth]{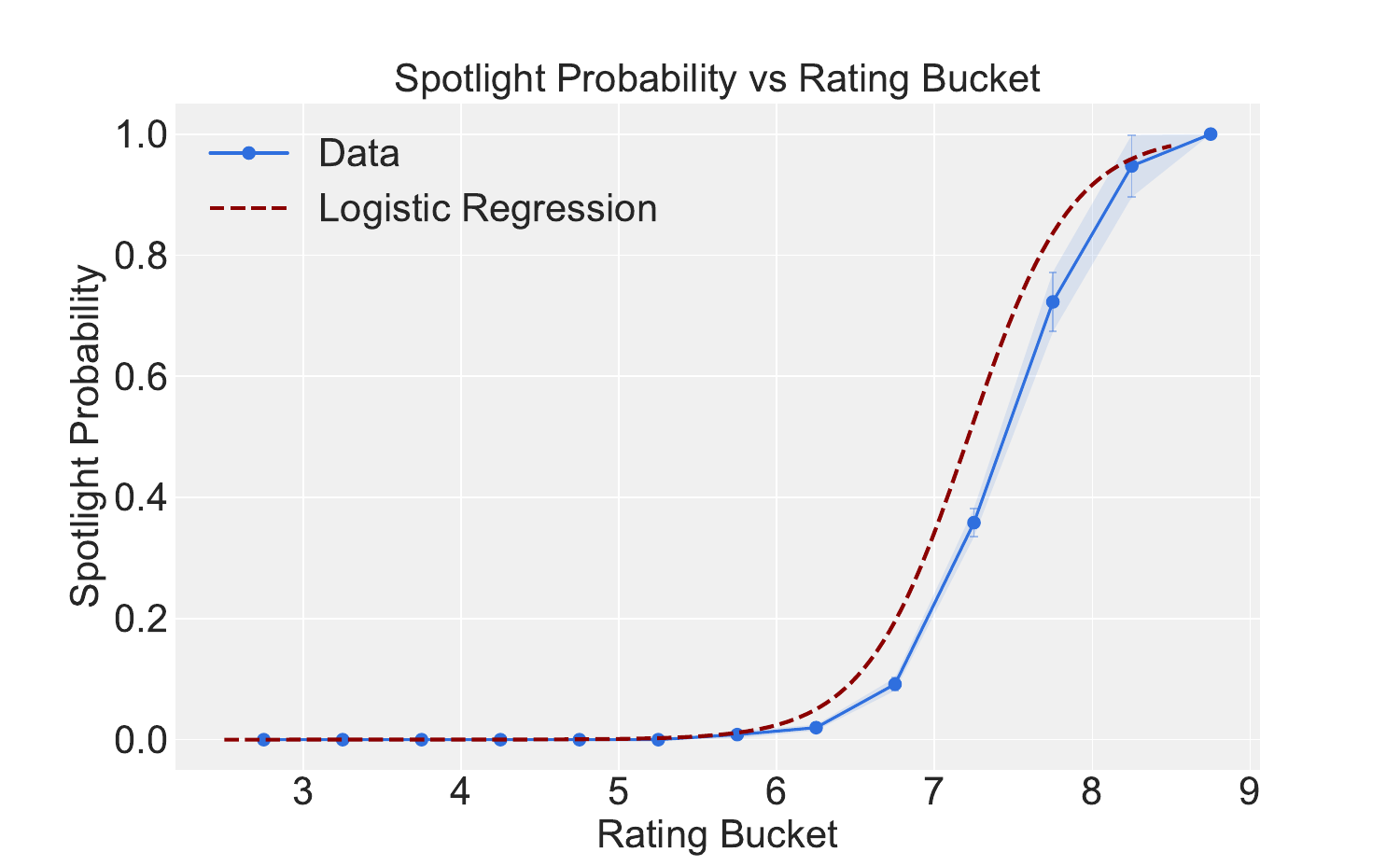}
    \caption{Spotlight (Highlighted) Probabilities of NeurIPS 2021}
    \label{fig:oral_plot2}
  \end{subfigure}
  \hfill
  \begin{subfigure}[b]{.45\textwidth}
    \centering
    \includegraphics[width=1.2\linewidth]{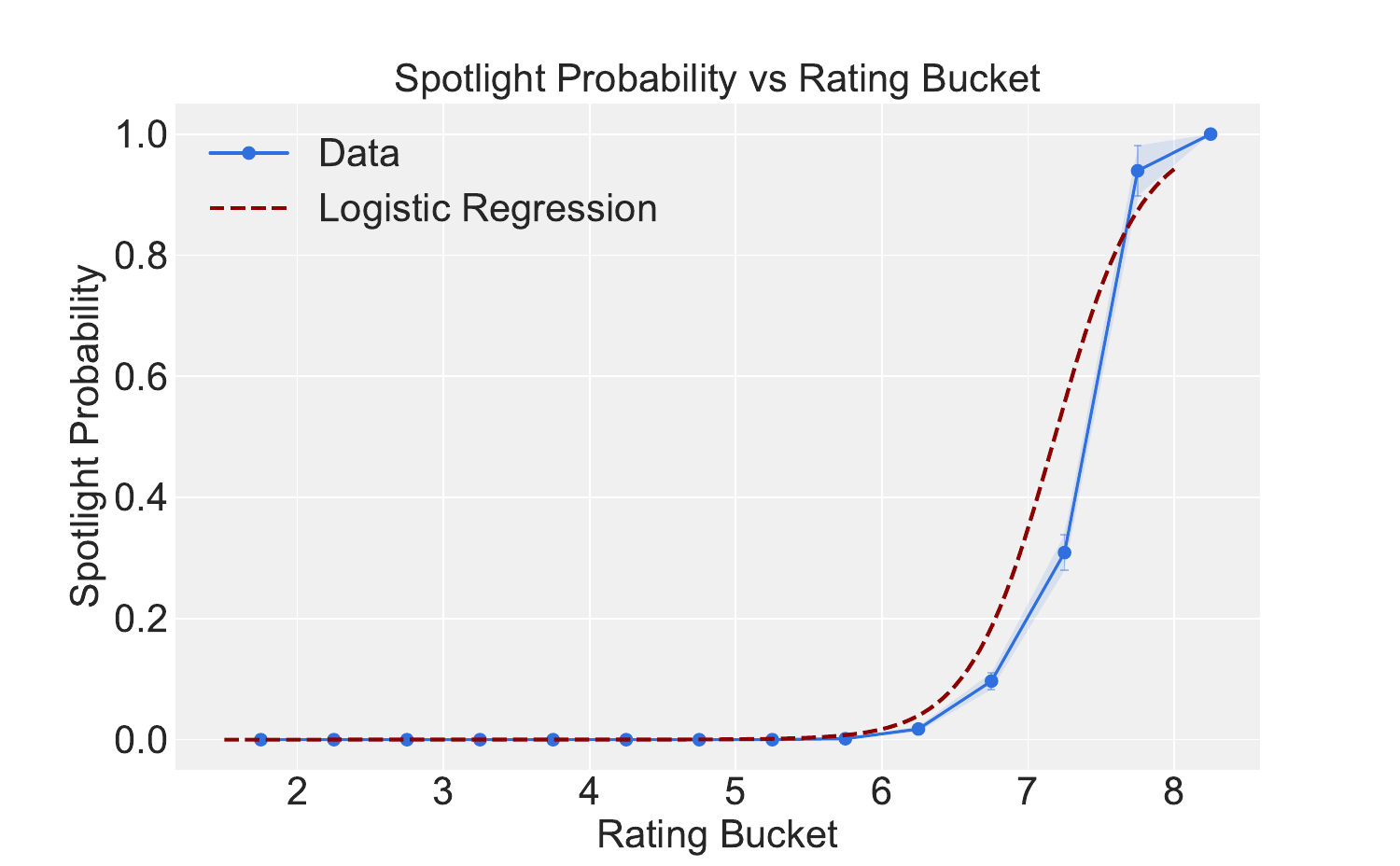}
    \caption{Spotlight (Highlighted) Probabilities of NeurIPS 2022}
    \label{fig:oral_plot3}
  \end{subfigure}

  \caption{Comparison of Spotlight (Highlighted) Probabilities for NeurIPS 2021 and 2022}
  \label{fig:combined_spotlight_probabilities}
\end{figure}

Finally, we shall also see a convex nature for the conditional probability concerning the ``best paper'' category in Figures \ref{fig:bestlog4}, \ref{fig:second2}. Such a convex behavior underscores the increasing likelihood of a paper being recognized as the ``best paper'' as its rating improves in NeurIPS 2021---2023.
\\\vskip -0.1in
\textbf{Remark: } To ensure the availability of sufficient data for our analysis, we have broadened the definition of ``best paper'' to encompass all paper awards. For instance, in the case of NeurIPS 2023 \citep{chairs2023AnnouncingNeurIPS20232023}, we included various categories under this term. This encompasses not just the traditional ``best paper'' award but also extends to ``Outstanding Main Track Papers'', ``Outstanding Main Track Runner-Ups'', and ``Outstanding Datasets and Benchmark Track Papers''.

\section{Simulation of the Best Paper Recommending Procedure}\label{prod}
Section \ref{sim} details a practical procedure for best paper selection, adapting the multi-author Isotonic Mechanism of \citep{wuIsotonicMechanismOverlapping2023} to the specific challenges of best paper awards. To this end, we introduce novel structural elements crucial for this context, namely the design of author \emph{quotas} and a distinction between \textbf{Blind} and \textbf{Informed} evaluation cases. 
We then evaluate the effectiveness of this new, context-aware procedure through simulation using synthetic data in Section \ref{synthetic}.

Real-world deployment is our ultimate objective, yet it cannot precede community endorsement: ground-truth paper quality is never directly observable, and access to the live review pipeline requires the field's explicit approval—an approval that hinges on publication. Before that, synthetic data allows us to vary authorship density, productivity dispersion, and noise in a controlled yet realistic fashion.

\subsection{Isotonic Mechanism under Multi-Ownership}\label{sim}

Let us assume an academic conference with $M$ authors and $N$ papers.
First, each paper $i \in [N]$ is assigned a noisy score $y_i = R_i + \epsilon_i$, where $R_i$ is the true score and $\epsilon_i$ reflects noise from the review process. For each author $j \in [M]$, let $I_j$ be the set of papers on which $j$ is a coauthor, and $\pi^j$ be the ranking the author provides for those papers.       

For any subset of papers $P_i \subset [N]$, denote by $\boldsymbol{y_{P_i}}$ the vector of noisy scores corresponding to $P_i$, and let $\pi^j_{P_i}$ be the restriction of $\pi^j$ to $P_i$. We then solve
\begin{equation}\label{equ:general15}
\begin{aligned}
    \min_{\boldsymbol{r}} \quad & \|\boldsymbol{y_{P_i}} - \boldsymbol{r}\|^2, \\
    \text{s.t.}\quad & \boldsymbol{r} \,\in\, S_{\pi_{P_i}^j},
\end{aligned}
\end{equation}
to obtain the isotonic solution $\boldsymbol{\hat{R}_{\pi_{P_i}^j}(y_{P_i})}$.
 
\noindent
\textbf{Overall Procedure:}
\begin{enumerate}
    \item All authors submit their papers and provide rankings for papers they coauthor. This yields an author set $[M]$, a paper set $[N]$, and the associated structures described above.
    \item A \emph{partition mechanism} is used to generate a 1-strong partition of the paper set:
    \begin{algorithm}[H]\label{alg:partition}
    \caption{Greedy Algorithm for a 1-Strong Partition \citep{suTruthfulOwnerAssistedScoring2022a,wuIsotonicMechanismOverlapping2023}}
    \begin{algorithmic}[1]
        \STATE \textbf{Input:} Ownership sets $\{I_j\}_{j \in [M]}$
        \STATE \textbf{Initialize} a partition $P = \emptyset$ and a set of selected papers $I = \varnothing$.
        \WHILE{$I \neq [N]$}
            \STATE Identify the largest set $P^* = I_{j^*} \setminus I$, where $j^* = \operatornamewithlimits{argmax}_{j \in [M]} |I_j \setminus I|$.
            \STATE Update $I \leftarrow I \cup P^*$ and $P \leftarrow P \cup \{P^*\}$.
        \ENDWHILE
        \STATE \textbf{return} $P$.
    \end{algorithmic}
    \end{algorithm}
    \item Given the partition $P = \{P_1,P_2,\cdots, P_l\}$, we determine $$T_i :=\{\textrm{authors who completely own } P_i\}.$$ Then, for each author $j\in T_i$, we get $\boldsymbol{\hat{R}_{\pi_{P_i}^j}(y_{P_i})}$ according to \eqref{equ:general15}. We use \begin{equation}
	\sum_{j \in T_i}\frac{1}{|T_i|}\boldsymbol{\hat{R}_{\pi_{P_i}^j}(y_{P_i})},
\end{equation}as the final adjusted scores for papers in the partition $P_i$. Run for all partitions and we get the final adjusted scores $\boldsymbol{\hat{R}}$.
    \item Finally, select the \emph{quota} of best papers (per author) for consideration, using authors’ own rankings. Decision-makers then follow one of the two approaches below:
    
       \noindent \textbf{Blind Case:} In this case, the decision-makers shall just review all those selected papers, given their original scores $\boldsymbol{y}$ and the adjusted scores $\boldsymbol{\hat{R}}$, unaware of each paper's relative ranking.

      \noindent  \textbf{Informed Case:} In this case, the decision-maker has full knowledge of the authors' relative rankings of the $k$ papers they have recommended, and also given those papers' original scores $\boldsymbol{y}$ and the adjusted scores $\boldsymbol{\hat{R}}$.  Consequently, the decision-maker can prioritize reviewing papers that are ranked higher. For instance, the decision-maker might commence by reviewing papers that are ranked first by certain authors. If the selection process for the best papers isn't complete after this, they would then review papers that are second in rankings of some authors, continuing this pattern until all the best papers are chosen. By employing this method, the decision-maker can possibly focus their efforts and save time. Additionally, this approach can be seen to approximately align with the assumptions of our theorem \ref{thm:betteru}, ensuring that authors remain truthful.
\end{enumerate}

\begin{table}[t]
    \caption{Practical application conditions for the Isotonic Mechanism. A check mark (\checkmark) indicates applicability under verified assumptions, while a smiley face (\smiley) indicates that no convexity assumption is required. The \emph{quota} is denoted by $k$. Specifically, \smiley\ corresponds to the condition that $U$ (or $U_1$) is nondecreasing, and \checkmark\ corresponds to the condition that $U$ is convex and nondecreasing. The \textbf{unlimited-quota} case implies the \emph{quota} is larger than any author's paper count (i.e., $k \ge n$), as discussed in Section \ref{sec:expe}.}
    \label{tab:prac}
    \centering
    \begin{small}
 \begin{tabular}{cccc}
    \toprule
     & $k$ = 1 & $2 \le k \le n-1$ & \textbf{unlimited-quota} \\
    \midrule
    \textbf{Blind Case} & \smiley & $U$ convex nondecreasing & \checkmark \\
    \midrule
    \textbf{Informed Case} & \smiley & $U_{k}$ convex nondecreasing, $U_{k}' \ge U_{k+1}'$ & $U_{k}$ convex nondecreasing, $U_{k}' \ge U_{k+1}'$ \\
    \bottomrule
    \end{tabular}
    \end{small}
\end{table}

We emphasize two practical points for real-world deployment.  
First, the procedure never disadvantages single-paper authors: an author with only one submission simply nominates that paper for best paper consideration. The adjusted scores merely serve as a recommendation, ensuring no penalty for single-paper submissions.
Second, the greedy partition algorithm yields several 1-strong blocks, ensuring that each author’s ranking is processed in isolation.  Because rankings from different blocks are not merged, this design sidesteps the co-author conflict issue discussed by \citep[Sec.\,3.2]{wuIsotonicMechanismOverlapping2023}.

Along with the \textbf{Assumption 1}, and the relaxed \textbf{Assumption 2} as demonstrated above, papers \citep{suTruthfulOwnerAssistedScoring2022a, wuIsotonicMechanismOverlapping2023} and this paper give the theoretical guarantee for this whole procedure to work in realistic cases. Note that, we recommend the \textbf{Blind Case} since it is also difficult to validate the conditions of the \textbf{Informed Case} in Theorem \ref{thm:betteru} (this is also a possible future direction to explore). Lastly, we want to highlight another assumption that is subtly implied by our models:
\begin{assumption}
    Every paper has a true score $R$, shared among all of its authors, and each author knows the true scores of the papers they coauthor.
    \label{asp:score}
\end{assumption}

\subsection{Experimental Evaluation}
\label{synthetic}

Having established the incentive compatibility of the Isotonic Mechanism, we now evaluate its practical effectiveness. This section investigates whether the mechanism can leverage authors' honest rankings to improve the selection of high-quality papers. To this end, we conduct a series of experiments on synthetic conference data where ground-truth paper quality is known. To determine the efficacy of different selection techniques, we measure the average quality of the papers they select. Our methodology, detailed in Appendix~\ref{expsetup}, assumes that co-authors know their paper's true quality score, in alignment with \textbf{Assumption~\ref{asp:score}}.

Our simulation pipeline consists of three core components: an authorship network, a paper quality model, and a review score generator. We explore two distinct models for the \textbf{authorship network}. The \emph{Uniform} model represents a dense and evenly distributed network, creating an ideal scenario for collaboration. Conversely, the \emph{ICLR} model is an exact replica of the ICLR 2021 co-authorship network \citep{ICLR2021}, which allows us to emulate the sparsity and heterogeneity of a real-world setting.

We also use two approaches for modeling \textbf{paper quality}. In the \emph{standard} approach, a paper's quality is determined by the maximum quality score among its authors, with additional random noise. In the \emph{productivity-weighted} approach, authors also receive a small quality bonus based on their total number of submissions, reflecting the idea that prolific researchers often gain experience and produce higher-quality work. Finally, the \textbf{review scores} are created by adding i.i.d. normally distributed noise to each paper's true quality score. 

In our experiments, we task the selection mechanisms with identifying the top $F$ papers, testing scenarios for both a single winner ($F=1$) and multiple winners ($F=10$). The author \textbf{quota}, $k$, determines the set of papers eligible for an award. 
Different choices of the $k$ (\emph{quota}) value lead to different sets of top papers. Although $k=1$ already identifies the papers all authors consider their best, theoretically, varying $k$ could have a slight impact on selecting the best paper, especially in edge cases (and best-paper selection should respect edge cases; see Appendix~\ref{app:edge}). But, as seen in the full data in Appendix~\ref{expsetup}, this impact is not substantial. It's also important to note that in our truthfulness Table~\ref{tab:prac}, $k=1$ has the strongest theoretical guarantees. Therefore, the scenario with $k=1$ is actually the most important, and we present results primarily for this case, relegating other values of $k$ to Appendix~\ref{expsetup}.

With the synthetic data generated, we must now define how to use the Isotonic Mechanism's output to select top papers. It is important to clarify that the mechanism's primary role is to offer more accurate paper scores by leveraging truthful author rankings. It does not, however, dictate a specific method for selecting the ``best'' paper. In practice, this decision often goes beyond mere scores, encompassing qualitative factors and a thorough review of the candidate papers. Therefore, for our evaluation, we will define several concrete selection protocols to benchmark the effectiveness of using these adjusted scores.

To this end, we define and evaluate the following three selection protocols:

\begin{itemize}
    \item \textbf{Benchmark}: Selects the $F$ papers with the highest raw review scores (that is, the noisy scores), ignoring all information from the Isotonic Mechanism.

    \item \textbf{Blind}: First, this protocol partitions the authorship network and runs the Isotonic Mechanism to generate adjusted scores. It then creates a candidate pool of all papers ranked within the top $k$ by at least one of their authors. From this pool, it selects the $F$ papers with the highest adjusted scores. As the name suggests, this protocol operates under the \textbf{Blind Case}, since the decision-maker does not use the specific author rankings, only the final scores.

    \item \textbf{Informed, Max}: This protocol begins with the same steps as the Blind Protocol. However, it sorts the candidate papers primarily by their worst numerical rank received from any co-author, in increasing order. Ties are then broken using the adjusted scores. Finally, it selects the top $F$ papers from this sorted list. This protocol operates under the \textbf{Informed Case}, as it explicitly uses the ranks provided by authors in its selection logic.
\end{itemize}

The distinction between the \textbf{Blind} and \textbf{Informed} protocols is critical. The Blind Protocol is straightforward and robust, relying only on the final adjusted scores. In contrast, the theoretical guarantees for the \textbf{Informed Case} hinge on a utility function with a specific asymmetric structure, where higher-ranked papers are assumed to provide qualitatively greater utility (as detailed in Table~\ref{tab:prac}).

Designing a practical selection protocol that perfectly mirrors this theoretical asymmetry is a significant challenge. Our \textbf{Informed, Max} protocol represents one such attempt, specifically designed to prioritize papers that all co-authors unanimously rank highly. As our results will demonstrate, this approach highlights the complexities of the \textbf{Informed Case}; its strictness can be a liability in certain network structures (although it may also yield benefits in others). Given these design challenges, we recommend the \textbf{Blind}  Protocol for its fairness and robust performance in practical deployment. Developing more sophisticated mechanisms for the \textbf{Informed Case} remains a promising direction for future research.

To evaluate the performance of these protocols, we use a normalized quality score: the average ground-truth quality of the $F$ papers chosen by a protocol, divided by the average quality of the actual top $F$ papers in the synthetic dataset. Figures~\ref{fig:combined_iclr_methods}, \ref{fig:combined_uniform_methods}  and Figures \ref{fig:bestlog21}--\ref{fig:second12} plot the mean of this metric with error bands  showing $\pm 1$ standard error of the mean (SEM), where SEM is estimated from 200 independent simulation rounds.

Before presenting the results, it is worth noting the mechanism's broad impact. The rankings provided by prolific ``hub'' authors, who have many submissions, induce long total orders that allow the mechanism to adjust a large number of scores. In our ICLR 2021 network replication, for instance, roughly \textbf{34\%} of all review scores were modified, indicating that the mechanism has a substantial footprint on the final data (and this proportion will rise as collaboration becomes more common).

\paragraph{Results.}

Our experiments show that the \textbf{Blind} Protocol consistently and significantly outperforms the benchmark across all tested settings. This performance improvement becomes more pronounced as the variance of the review error increases, highlighting the mechanism's value in noisy environments. The gains are larger in the dense, Uniform network than in the sparser ICLR network, which is expected since denser collaboration networks provide more ranking information to the mechanism. Furthermore, the improvement is amplified in the productivity-weighted model. This is likely because it concentrates high-quality papers among prolific authors, whose comprehensive rankings allow the mechanism to more effectively adjust the scores of top-tier papers. Crucially, a significant improvement remains even in the standard ICLR network, demonstrating the robust benefits of the \textbf{Blind} Protocol.

The \textbf{Informed, Max} exhibits more context-dependent performance. Although it is often successful, its effectiveness can degrade in specific scenarios. Notably, within the ICLR network combined with a productivity bonus and for a larger award pool ($F=10$), its performance can fall below the benchmark. This outcome is analyzed in detail in Appendix~\ref{expsetup}.

\paragraph{Discussion.}

The divergent performance of the \textbf{Blind} and \textbf{Informed} protocols offers a key insight: there is a trade-off between theoretical fidelity and practical robustness. The \textbf{Informed} Protocol was designed to more faithfully adhere to the theoretical requirement for an asymmetric utility function. However, its reliance on a strict, rank-based sorting rule makes it brittle; it can be overly sensitive to the complex co-authorship patterns found in real-world networks.

The \textbf{Blind} Protocol, in contrast, proves to be a more resilient strategy. It uses author rankings solely to refine the accuracy of the scores, leaving the final, simple selection to be based on these improved scores. This simpler approach is consistently effective across all our simulated conditions. Therefore, we recommend the \textbf{Blind Protocol} for practical implementation due to its strong and reliable performance. The design of more sophisticated and adaptive protocols for the \textbf{Informed Case}—ones that can leverage full ranking information without sacrificing robustness—presents a compelling avenue for future research.

\begin{figure}[h]
  \centering

  \begin{subfigure}[b]{.45\textwidth} 
    \centering
    \includegraphics[width=1.2\linewidth]{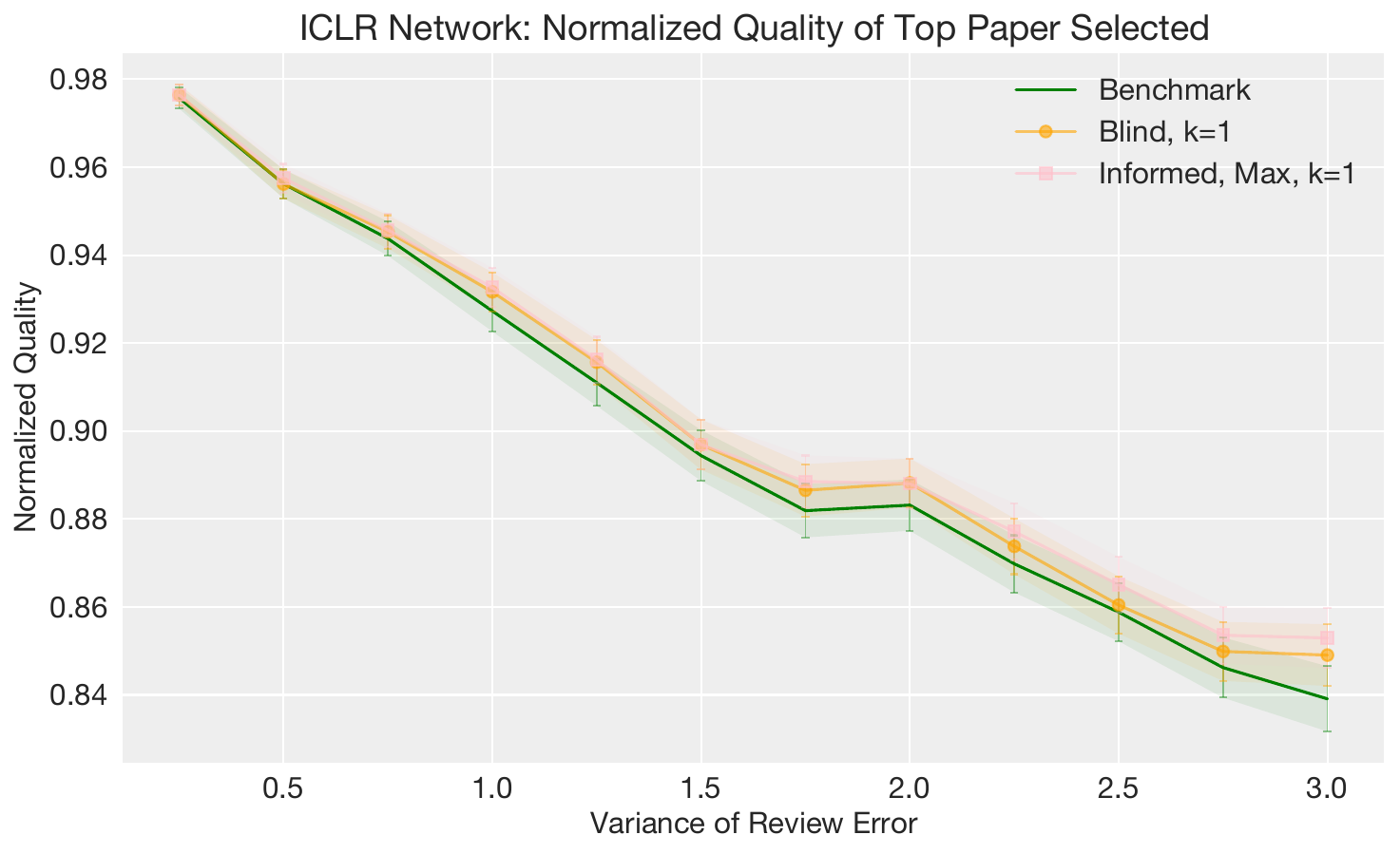}
\caption{Single best paper selection comparison, ICLR Network}
    \label{fig:n1ICLRsimple}
  \end{subfigure}
  \hfill 
  \begin{subfigure}[b]{.45\textwidth} 
    \centering
    \includegraphics[width=1.2\linewidth]{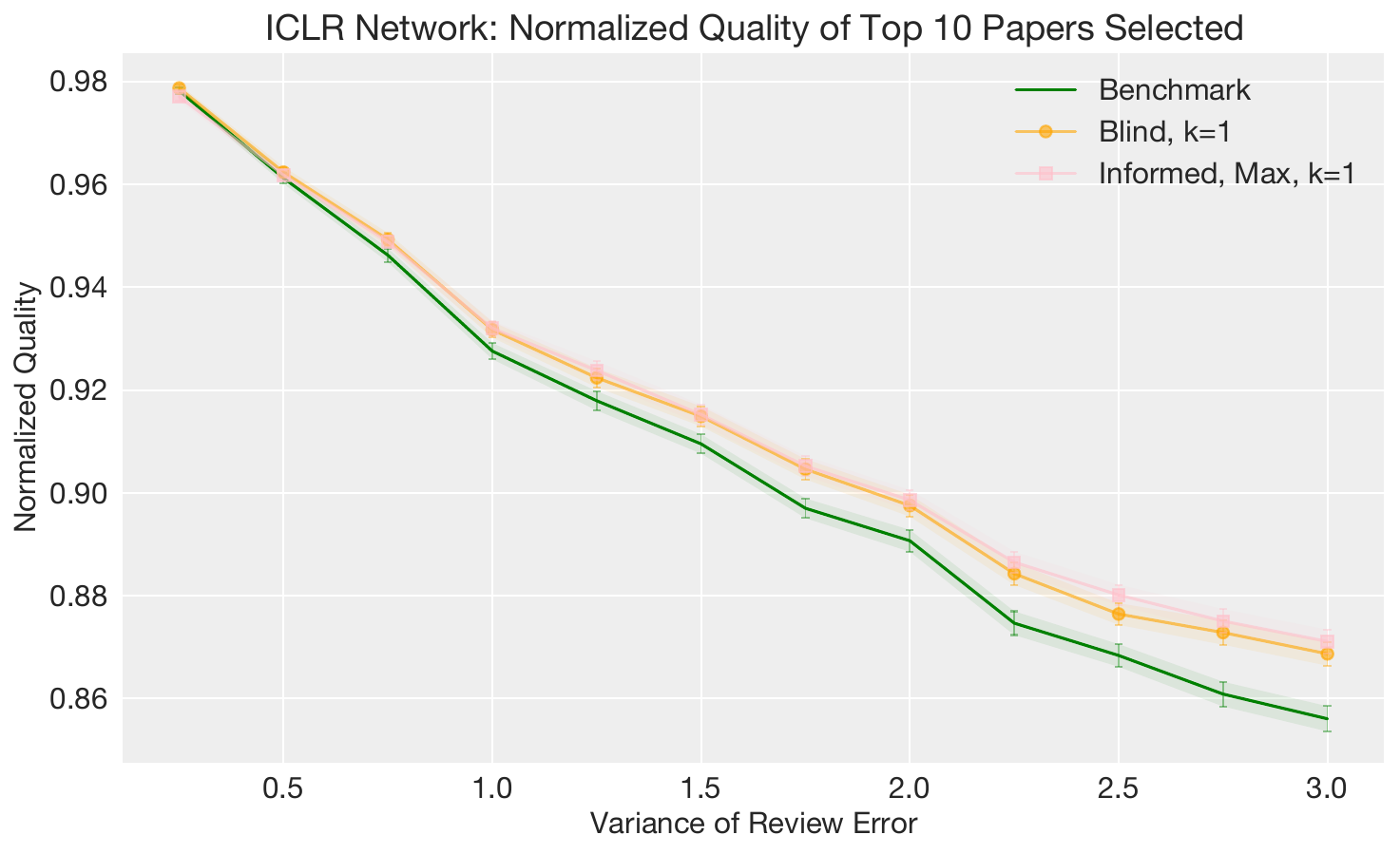}
\caption{10 best papers selection comparison, ICLR Network}
    \label{fig:n10ICLRsimple}
  \end{subfigure}

  \vspace{1em} 
  
  \begin{subfigure}[b]{.45\textwidth}
    \centering
    \includegraphics[width=1.2\linewidth]{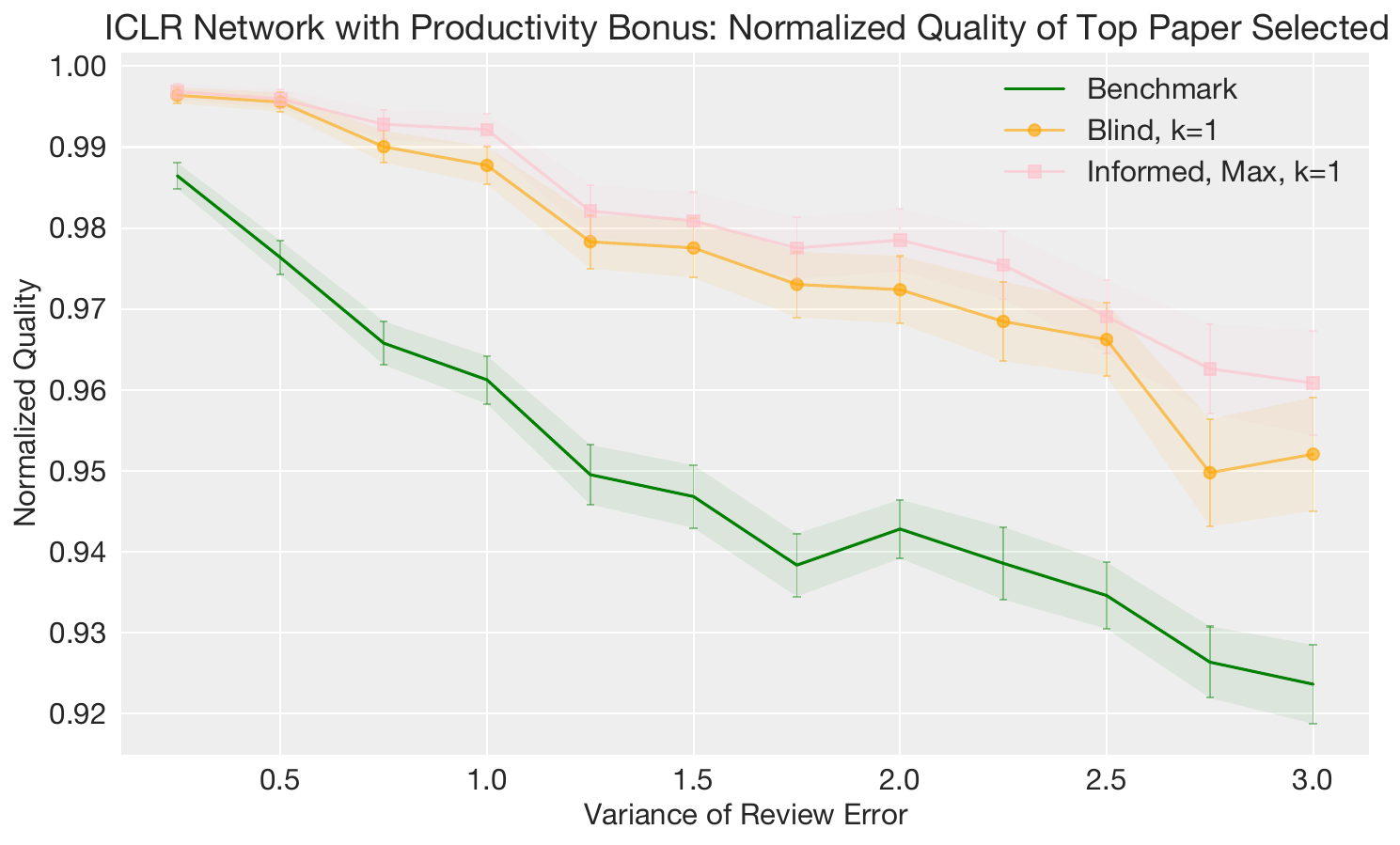}
\caption{Single best paper selection comparison, ICLR Network with Productivity Bonus}
    \label{fig:n1ICLRprodsimple}
  \end{subfigure}
  \hfill 
  \begin{subfigure}[b]{.45\textwidth}
    \centering
    \includegraphics[width=1.2\linewidth]{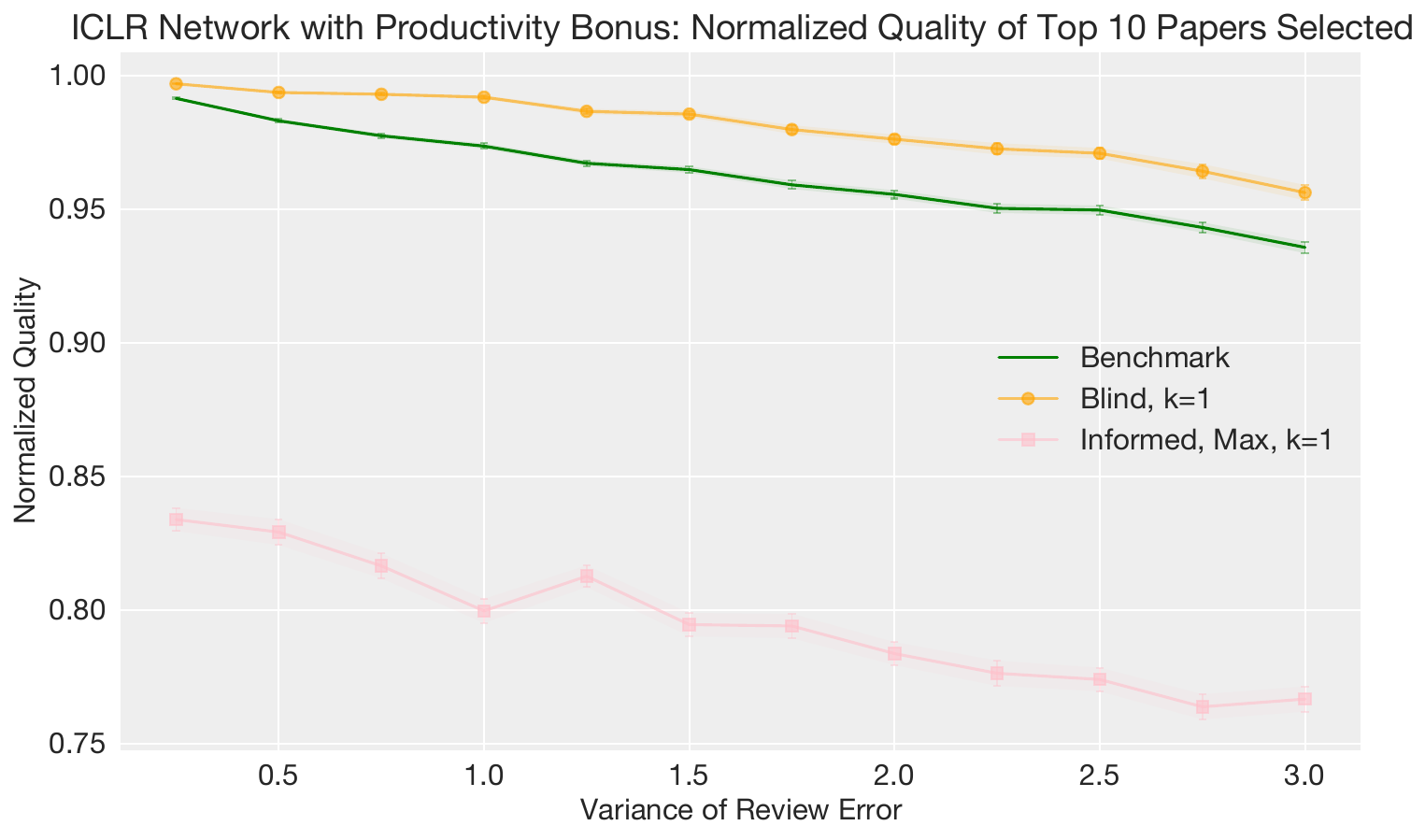}
   \caption{10 best papers selection comparison, ICLR Network with Productivity Bonus (see Appendix~\ref{note})}
    \label{fig:n10ICLRprodsimple}
  \end{subfigure}

\caption{Performance comparison of paper selection methods in ICLR Network. Shaded bands denote $\pm 1$ standard error of mean.} 
  \label{fig:combined_iclr_methods}
\end{figure}

\begin{figure}[h]
  \centering

  \begin{subfigure}[b]{.45\textwidth}
    \centering
    \includegraphics[width=1.2\linewidth]{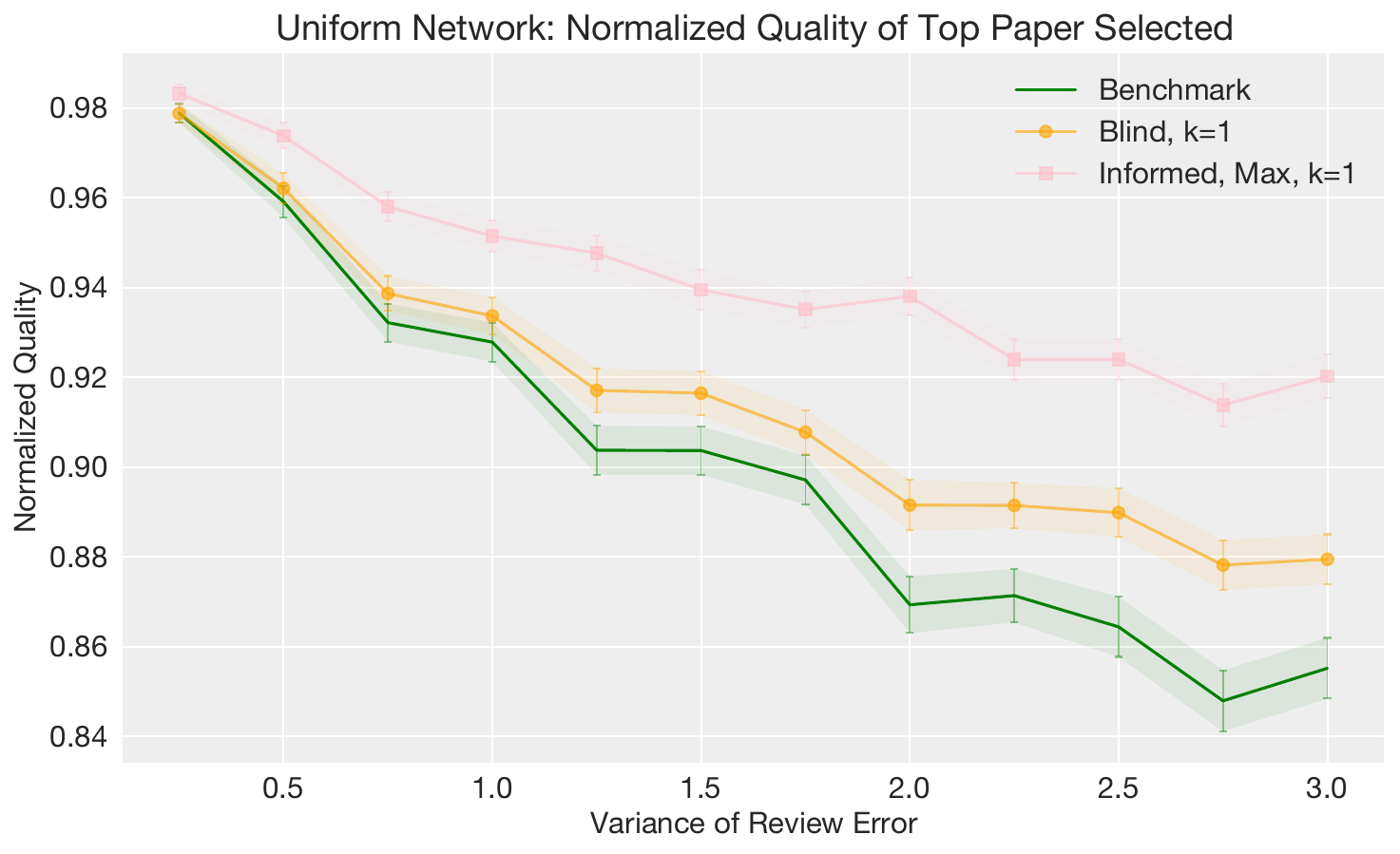}
    \caption{Single best paper selection comparison, Uniform Network}
    \label{fig:n1uniformsimple}
  \end{subfigure}
  \hfill 
  \begin{subfigure}[b]{.45\textwidth}
    \centering
    \includegraphics[width=1.2\linewidth]{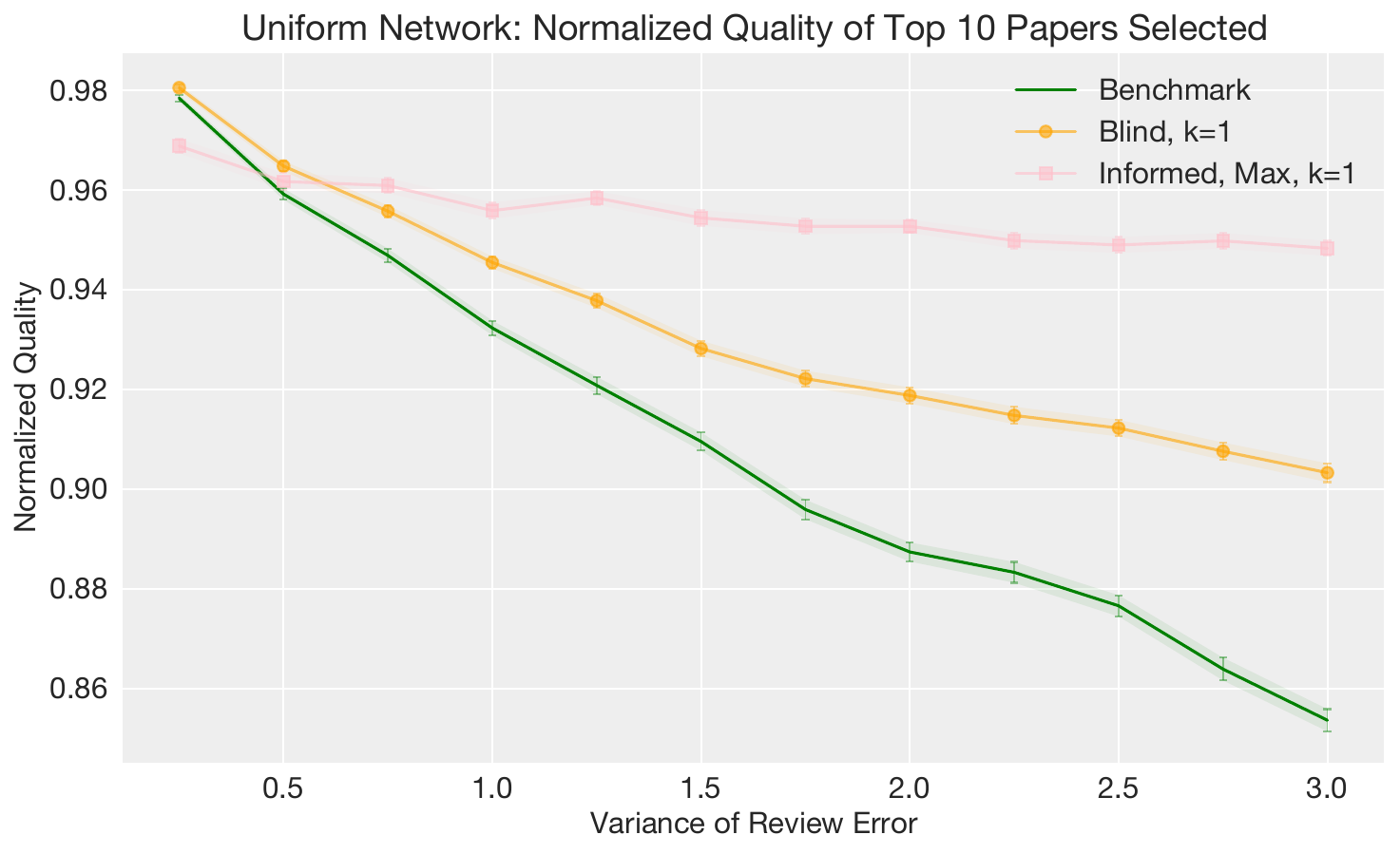}
    \caption{10 best papers selection comparison, Uniform Network}
    \label{fig:n10uniformsimple}
  \end{subfigure}

  \vspace{1em} 
  
  \begin{subfigure}[b]{.45\textwidth}
    \centering
    \includegraphics[width=1.2\linewidth]{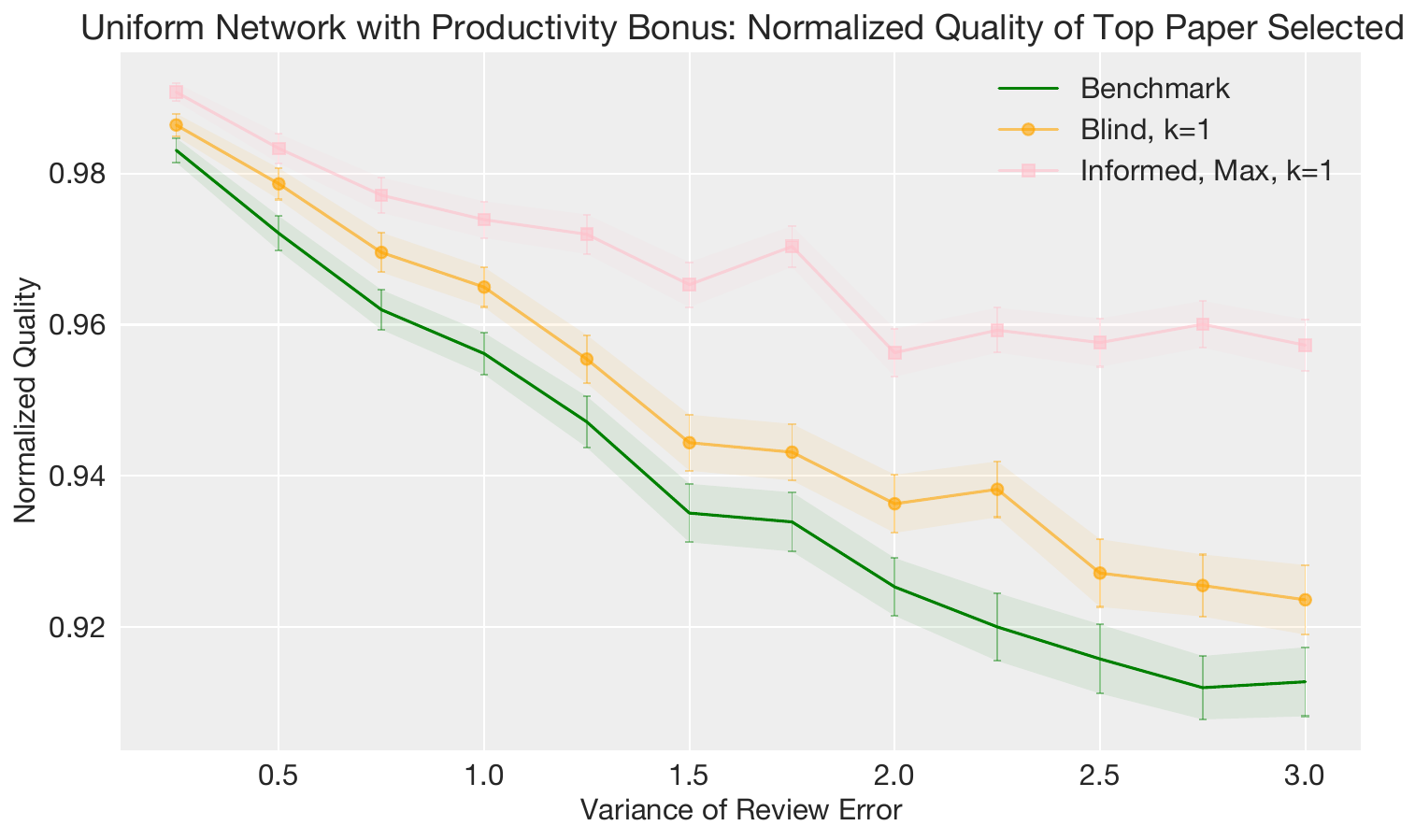}
    \caption{Single best paper selection comparison, Uniform Network with Productivity Bonus}
    \label{fig:n1uniformprodsimple}
  \end{subfigure}
  \hfill 
  \begin{subfigure}[b]{.45\textwidth}
    \centering
    \includegraphics[width=1.2\linewidth]{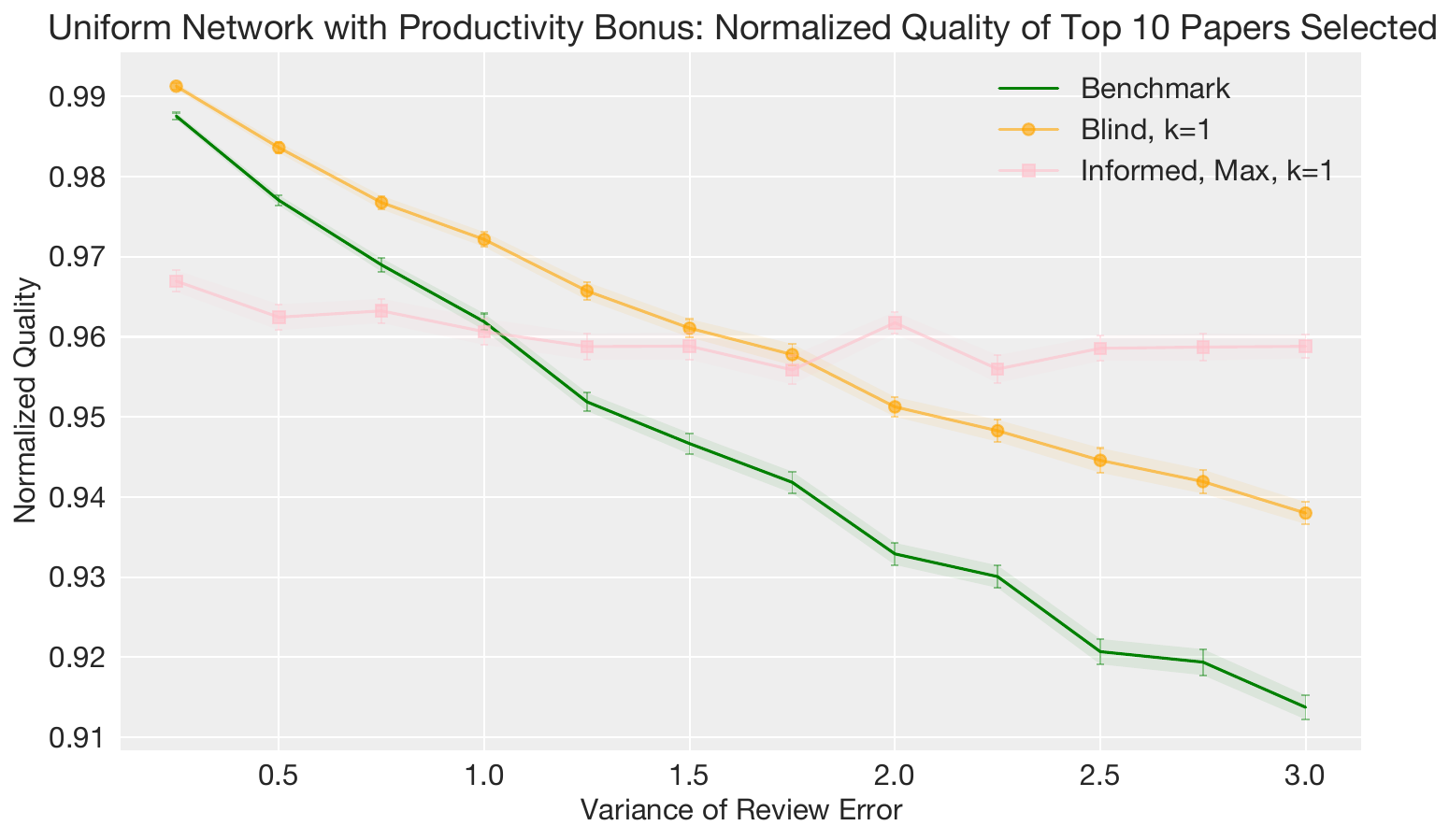}
    \caption{10 best papers selection comparison, Uniform Network with Productivity Bonus  (see Appendix~\ref{note})}
    \label{fig:n10uniformprodsimple}
  \end{subfigure}

  \caption{Performance comparison of paper selection methods in Uniform Network.}
  \label{fig:combined_uniform_methods}
\end{figure}

\section{Conclusions}
\label{sec:conclusion}

In response to the growing challenges of maintaining a fair and rigorous peer review process, this paper develops an author-assisted, incentive-compatible mechanism for selecting best paper awards. Our approach enhances the Isotonic Mechanism framework to be more robust and practical for real-world deployment by carefully relaxing and empirically validating its convexity assumptions. Our result is summarized in Table~\ref{tab:prac}.
 The resulting procedure not only demonstrates substantial gains in selection quality in simulations but also offers a transparent, verifiable, and readily applicable solution to improve a critical aspect of peer review. Below, we detail our primary contributions in relation to prior work before discussing limitations and future directions.

\subsection{Contributions in Relation to Prior Work}
\label{sec:contributions}

While our work builds upon the Isotonic Mechanism framework introduced by \citep{suTruthfulOwnerAssistedScoring2022a} and extended to overlapping authorship by \citep{wuIsotonicMechanismOverlapping2023}, it offers   key contributions that advance its theoretical underpinnings and practical applicability.

First, and most centrally, we identify a crucial setting where the convexity assumption required by prior work \citep{suTruthfulOwnerAssistedScoring2022a, wuIsotonicMechanismOverlapping2023, yanIsotonicMechanismExponential2023a} can be relaxed to mere monotonicity in the context of single best paper nomination (\emph{quota} of $1$). While the proof itself may appear intuitive, the conceptual leap it represents is nontrivial. Convexity is a much stronger and harder-to-verify assumption in subjective domains like peer review. By demonstrating that the mechanism's guarantees hold under the far milder and more defensible requirement of monotonicity, our results substantially lower the barriers to real-world adoption, making the mechanism more robust and readily applicable.

Second, we tailor the Isotonic Mechanism to the specific context of best paper awards. This involves introducing and analyzing unique structural elements, such as the design of \emph{quotas} and the critical distinction between the \textbf{Blind} and \textbf{Informed} cases. These considerations are vital for understanding implementation challenges in this new domain and are not addressed in the prior literature.

Finally, we provide new empirical validation for this specific use case. We conduct dedicated experiments, including on the real-world ICLR 2021 co-authorship network, to evaluate the mechanism's effectiveness in the best-paper selection setting (Section~\ref{prod}). Furthermore, we show that under the \textbf{unlimited-quota} case, the convexity assumption is indeed plausible and broadly supported by real conference data (Section~\ref{sec:expe}).
 
\paragraph{Generality Beyond Peer Review.}
In many domains, insiders (creators, managers, or applicants) possess high-fidelity ordinal information about the items they manage. 
Crucially, our theoretical results demonstrate that when the objective is restricted to \textbf{selecting the top-tier items} (as in best paper awards), the   convexity requirement on utility functions can be relaxed to mere monotonicity while improving selection accuracy.
This insight suggests that our framework is general enough to be deployed in varied high-stakes environments where identifying the best is the primary goal:

\begin{itemize}
    \item \textbf{Grant and Fellowship Allocation:} Applicants or research institutions often submit multiple proposals to funding bodies. These entities typically possess reliable relative assessments of their own submissions (knowing which proposal is their safest versus highest risk), while external review panels face high variance and heterogeneity.
    \item \textbf{Corporate Performance and Hiring:} In corporate settings, managers usually rank projects, teams, or internal proposals they oversee, providing a clear signal of relative merit even when absolute performance metrics are uncertain or difficult to calibrate across departments. Similarly, in hiring processes, recruiters or internal referrers often have partial but informative knowledge about the relative quality of a subset of candidates.
    \item \textbf{Crowdsourcing and Platforms:} The framework applies to human-in-the-loop systems where contributors provide relative judgments over tasks or model outputs. It is also relevant for platform ranking problems, such as recommendation systems or content moderation, where creators have partial knowledge about the relative quality of their own items.
\end{itemize}

\subsection{Limitations and Future Directions}
\label{sec:limitations}

\paragraph{Informed Case Optimization.}
Our experiments reveal that while the informed mechanism is generally effective, its performance can be context-dependent. Its effectiveness may decrease in complex scenarios, such as the ICLR network simulation with a productivity bonus when selecting ten best papers. A key future direction is to move beyond simple sequential selection rules and develop more advanced computational techniques. This could involve exploring decision-theoretic models or sophisticated algorithms designed to better handle the rich, asymmetric information provided by author rankings, thereby improving the mechanism's robustness and accuracy across a wider range of conditions.

\paragraph{Assumption Examination and Relaxation.}
Although our paper significantly relaxes the convexity assumption required by prior work \citep{suTruthfulOwnerAssistedScoring2022a,wuIsotonicMechanismOverlapping2023,yanIsotonicMechanismExponential2023a}, it still relies on several core assumptions, such as the additivity of the utility function and the exchangeability of review noise. 
A vital area for future theoretical work is to probe the boundaries of these assumptions. For instance, could the mechanism's truthfulness guarantees be preserved under different, non-additive utility models? Relaxing these foundational assumptions could significantly broaden the mechanism's applicability and resilience.

\paragraph{Real-world Validation and Quota Design.}
While our simulations, including those on a real-world co-authorship network, are promising, the ultimate test is a live deployment. A crucial next step is to implement and test the Isotonic Mechanism to recommend best papers in an actual conference setting. Such a real-world trial would provide invaluable feedback on the practical dynamics of author participation and allow for the calibration of our models against live data. Furthermore, although this paper primarily focuses on a uniform \emph{quota}, an open question remains as to how allowing different \emph{quota} values for individual authors (e.g. based on productivity) might affect fairness, truthfulness, diversity, and overall performance in the \textbf{Blind} and \textbf{Informed} settings.
\\ \vskip -0.1in
Addressing review quality degradation at major ML/AI conferences remains a critical community priority. Real-world deployment ultimately requires community endorsement and field approval---an approval achievable only through publication and adoption. We believe these advances hold promising implications for peer-review quality and the broader scientific community.
 
\section*{Acknowledgments}
{We would like to thank Jibang Wu and Xiang Li for their helpful discussions and comments. This
work was supported in part by NSF DMS-2310679, a Meta Faculty Research Award, and Wharton
AI for Business.}

\newpage
\appendix

\setcounter{theorem}{0}
\setcounter{assumption}{0}
 
\renewcommand{\thetheorem}{A.\arabic{theorem}}
\renewcommand{\theassumption}{A.\arabic{assumption}}

\section{Details about the Data Preparation of ICLR 2019-2023}
Due to the similarity in data analysis related to NeurIPS 2021-2023 with the process  described here, we shall omit its details for brevity.
\label{app:expe}
\subsection{Data Processing}
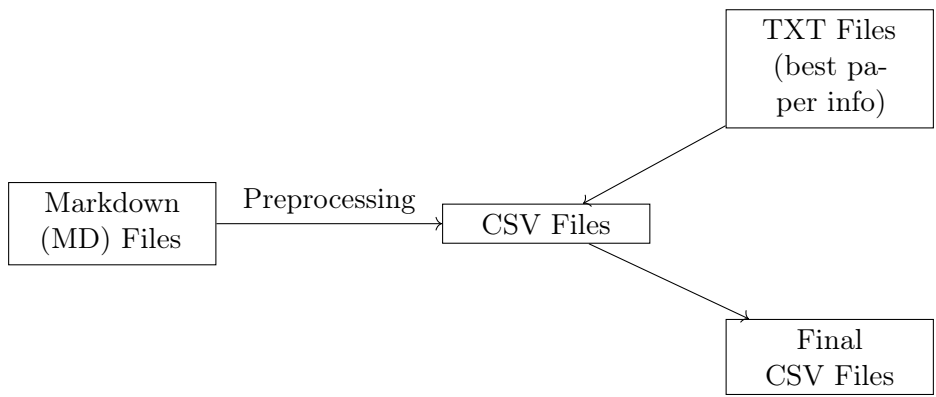
\begin{figure}[H]
\centering
\begin{tikzpicture}[
    node distance=2cm and 3cm,
    mynode/.style={rectangle, draw, align=center, text width=2.5cm}
]

\node[mynode] (md) {Markdown (MD) Files};
\node[mynode, right=of md] (csv) {CSV Files};
\node[mynode, above right=1cm and 1cm of csv] (txt) {TXT Files (best paper info)};
\node[mynode, below right=1cm and 1cm of csv] (csv_final) {Final CSV Files};

\draw[->] (md) -- node[midway, above, sloped] {Preprocessing} (csv);  
\draw[->] (csv) -- node[midway, above, sloped] {} (csv_final);
\draw[->] (txt) -- (csv);

\end{tikzpicture}
\caption{Process of data transformation from Markdown files to the final CSV files, incorporating the best paper information.}
\end{figure}

In this project, we have ten files for data analysis of ICLR 2019-2023: five CSV files and five corresponding TXT files \citep{ICLR2019a, syncedICLR2020Virtual2020a, iclrAnnouncingICLR2021a, brockmeyerAnnouncingICLR20222022a, wangAnnouncingICLR20232023}. The CSV files contain peer review data for the past five years, including information about ratings and decisions. The TXT files store the titles of the ``best paper'' articles for each year. The CSV files were obtained by preprocessing Markdown (MD) files \citep{sunCrawlVisualizeICLR2023c, sunCrawlVisualizeICLR2023d,zhouCrawlVisualizeICLR2023b, bertoCrawlVisualizeICLR2023b}, filtering and extracting the necessary information. This data organization enables easy access and analysis of data for each year, as well as a comparison between different years to identify trends or patterns. The data processing is totally the same for data analysis of NeurIPS 2021-2022, so we will omit its details here.

The processing of all the aforementioned data was performed using the file \texttt{dataProcess.ipynb.}

\footnote{\label{code:github}\href{https://github.com/pkumath/PeerReviewStat}{See https://github.com/pkumath/PeerReviewStat}}
\begin{itemize}
	\item We removed all papers from the original Markdown files that did not have scores assigned.
	\item During our data analysis, if a paper has a score but no review result, we consider it as being rejected.
\end{itemize}

\subsection{Data Analysis with Logistic Regression}
\subsubsection{Methodology}
We used logistic regression to predict the probability of desired outcomes (such as paper acceptance/rejection and best paper/not best paper) based on the averaged paper scores. We assumed that the review process for each paper is independent, which is a basic assumption of logistic regression for our analysis.

The main function used in our analysis is the \texttt{LogisticRegression} function from the \texttt{sklearn} package. We optimize the cross-entropy loss with an ${l_2}$ penalty for regularization. For more specific details, please refer to the code file \texttt{RegressionForAccept.ipynb}.

\section{Details on Experimental Results}\label{app:experiment}

\subsection{Experimental Setup}\label{expsetup}

In our experiment, we generated many samples of synthetic submission data, each consisting of an authorship network, ground truth quality scores for each paper, and reviews for each paper. Then, we measured the average true quality of the paper(s) selected by different best paper mechanisms. Below, we detail the different parameters in our setup. 

\paragraph{Authorship Network:}
We provide two models for the authorship network. The first, which we refer to as a uniform authorship model, has 2500 authors and 5000 papers. Each paper has a number of authors uniformly selected between 1 and 10, and these authors are selected uniformly at random from the author pool. The second, which we refer to as the ICLR model, has 2997 authors and 8956 papers. It is exactly the authorship network of ICLR 2021. The uniform model provides a simple and dense structure to highlight the benefits of the mechanism, while the ICLR model provides a specific real-world authorship network. 
\paragraph{Paper Quality:} 
In the standard model, each author is given a quality score $$q_{a} \sim \mathcal{N}(\mu = 5,\sigma^2 = 2)$$ Then, for a paper $p$ with a set of authors $A_{p}$, the quality of $p$ is 
$$q_{p} \sim \max_{a \in A_{p}}(q_{a}) + \mathcal{N}(\mu = 0,\sigma^2 = 1)$$

The standard model reflects the fact that the quality of a paper is greatly impacted by the quality of its authors. Note that we consider the contribution of author quality to paper quality to be the max of all authors, rather than the min, the mean, the sum, or the average. We think the highest-quality author on a paper is likely to be the most representative of the paper quality; respected authors do not want to put their names on low-quality work, while beginning researchers with (understandably) low research quality can provide assistance on these papers without having much impact on overall quality.

In the productivity-weighted model, an author $a$ with a set of papers $P_{a}$ is given a quality score 

$$q_{a} \sim \mathcal{N}(\mu = 5,\sigma^2 = 2) + \mathcal{N}(\mu = \frac{5}{16},\sigma^2 = \frac{2}{16^2})$$ 

The quality of each paper given the quality of the authors is equivalent to the standard model.

\paragraph{Review Error:}
Given a paper with true quality $q_{p}$, its review score is distributed according to 

$$r(q_p) \sim q_p + \mathcal{N}(\mu = 0, \sigma^2 = \epsilon)$$

Where the additional noise is i.i.d. and $\epsilon$ is the same amongst all papers. This was tested with various values for the variance of the noise, which we think of as review error.
\\

\paragraph{Quota  ($k$):} When using the Isotonic Mechanism, we consider only the top $k$-ranked papers from each author. We explore two settings: where $k=1$ and $k=5$. We find that in our experiments, increasing $k$ slightly decreases the quality of selected papers.

\paragraph{Number of Best Papers Winners ($F$):} 
Often, conferences select multiple best papers. Thus, we explore both the setting where the selection mechanisms select a single paper ($F=1$) and the top $10$ papers ($F=10$).
\\

\paragraph{Best Paper Selection Methods:}
In addition to the three best paper selection methods discussed in Section~\ref{synthetic}, we also tested a secondary method using information available in the \textbf{Informed Case}. This method prioritizes papers ranked early by at least one author, rather than prioritizing papers ranked early by all their authors. We restate the three previous selection methods and include a description of the additional selection method.

\begin{itemize}
\item \textbf{Benchmark}: Select the $F$ papers with the highest review scores (i.e. the noisy scores, ignoring any information from the Isotonic Mechanism).
\item \textbf{Blind}: Greedily partition the authorship network and run the Isotonic Mechanism to get adjusted scores. Then, consider only the set of papers $p$ for which at least one author ranked that paper in their top $k$. Select the top $F$ papers in $p$ with the highest adjusted scores. Note that this method is possible to operate in the \textbf{Blind Case}.
\item \textbf{Informed, Max}: Greedily partition the authorship network and run the Isotonic Mechanism to get adjusted scores. Then, consider only the set of papers $p$ for which at least one author ranked that paper in their top $k$. Sort $p$ in increasing order of their highest (i.e. worst) ranking among all authors. Tiebreak by sorting in decreasing order of adjusted score. Select the first $F$ papers. Note that this method is only possible to operate in the \textbf{Informed Case}.
\item \textbf{Informed, Min}: Greedily partition the authorship network and run the Isotonic Mechanism to get adjusted scores. Then, consider only the set of papers $p$ for which at least one author ranked that paper in their top $k$. Sort $p$ in increasing order of their lowest (i.e. best) ranking among all authors. Tiebreak by sorting in decreasing order of adjusted score. Select the first $F$ papers. Note that this method is only possible to operate in the \textbf{Informed Case}.
\end{itemize}

We find that the Informed mechanism using the Min ranking is never substantially different from the Blind mechanism. 

\subsection{Inconsistency of Informed Case with Max Ranking}\label{note}
In our data, we find that, while the max-ranking informed mechanism outperforms all other mechanisms in most models, in the $F=10$ ICLR network with productivity weighting in particular, it has the worst performance by far. This may be because this informed mechanism prioritizes selecting papers that are ranked first by every one of their authors. Thus, as long as there are at least $F$ such papers, it will never select two best papers with a shared author (because the author cannot have ranked them both first). Thus, in a network where a single author writes multiple of the true top $F$ papers, the informed mechanism will never be able to select more than one of them. This scenario is more likely in the productivity-weighted model than the standard model, because good papers will be more clustered amongst a smaller number of authors. Furthermore, this phenomenon is exacerbated in the ICLR network, because there is more variance in how many papers each author has (compared to the uniform network).

\subsection{Full Experimental Results}
\label{app:fullexp}
\begin{figure}[h]
  \centering
  \begin{minipage}{.45\textwidth}
    \centering
    \includegraphics[width=1.2\linewidth]{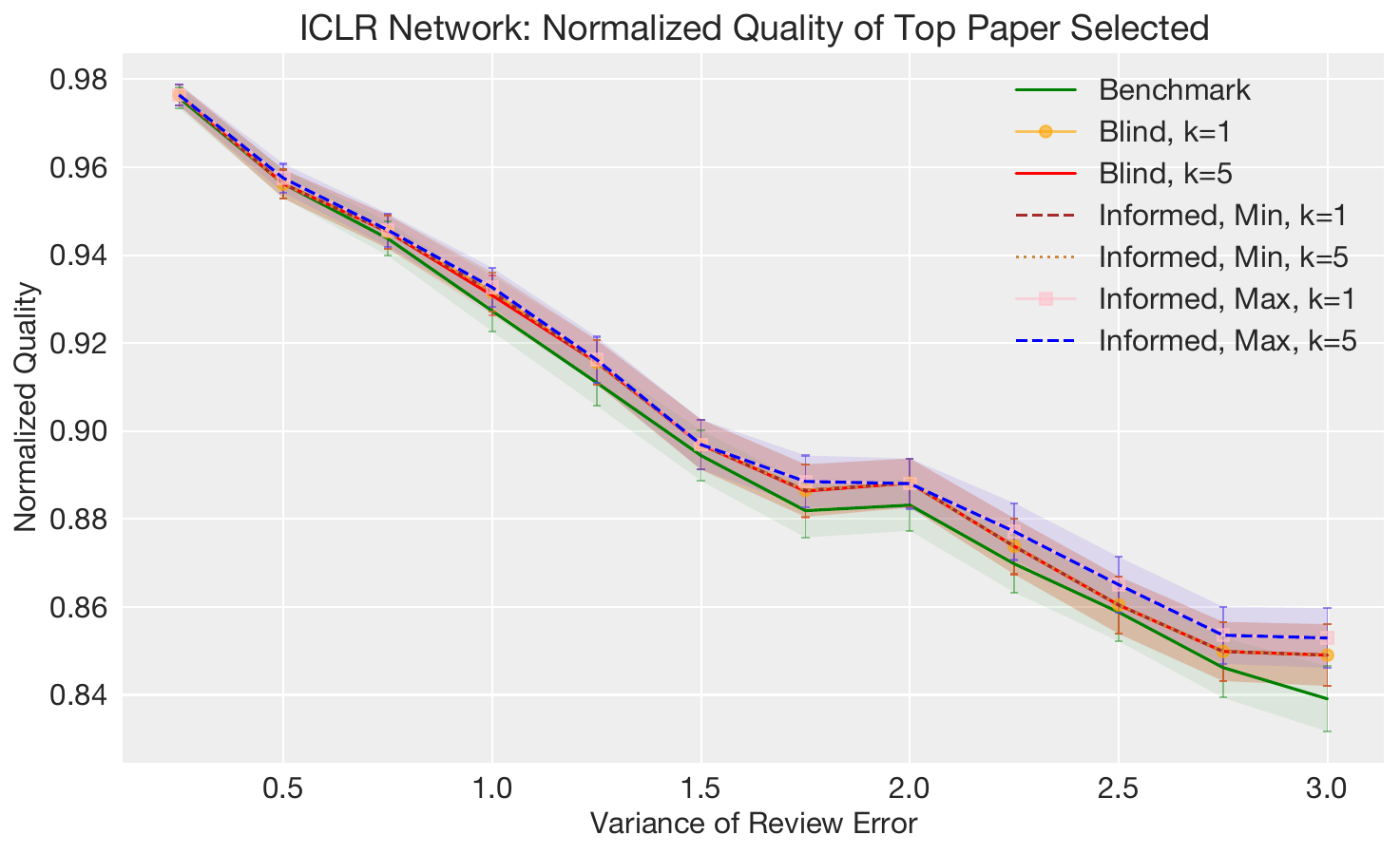}
    \caption{Comparison of Different Methods to select top paper, ICLR Network}
    \label{fig:bestlog21}
  \end{minipage}
  \hfill 
  \begin{minipage}{.45\textwidth}
    \centering
    \includegraphics[width=1.2\linewidth]{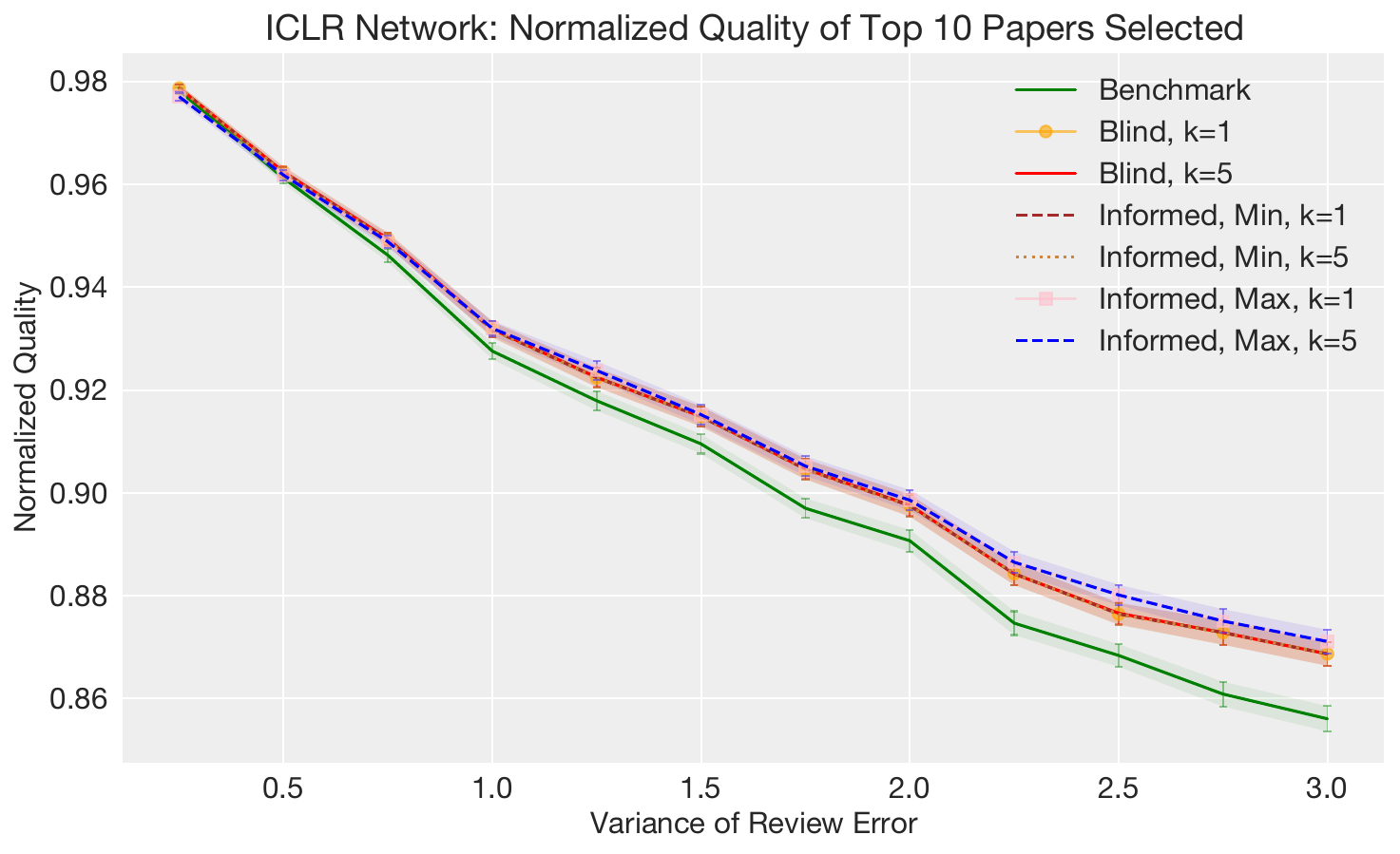}
    \caption{Comparison of Different Methods to select top 10 papers, ICLR Network}
    \label{fig:second12}
  \end{minipage}
\end{figure}

\begin{figure}[h]
  \centering
  \begin{minipage}{.45\textwidth} 
    \centering
    \includegraphics[width=1.2\linewidth]{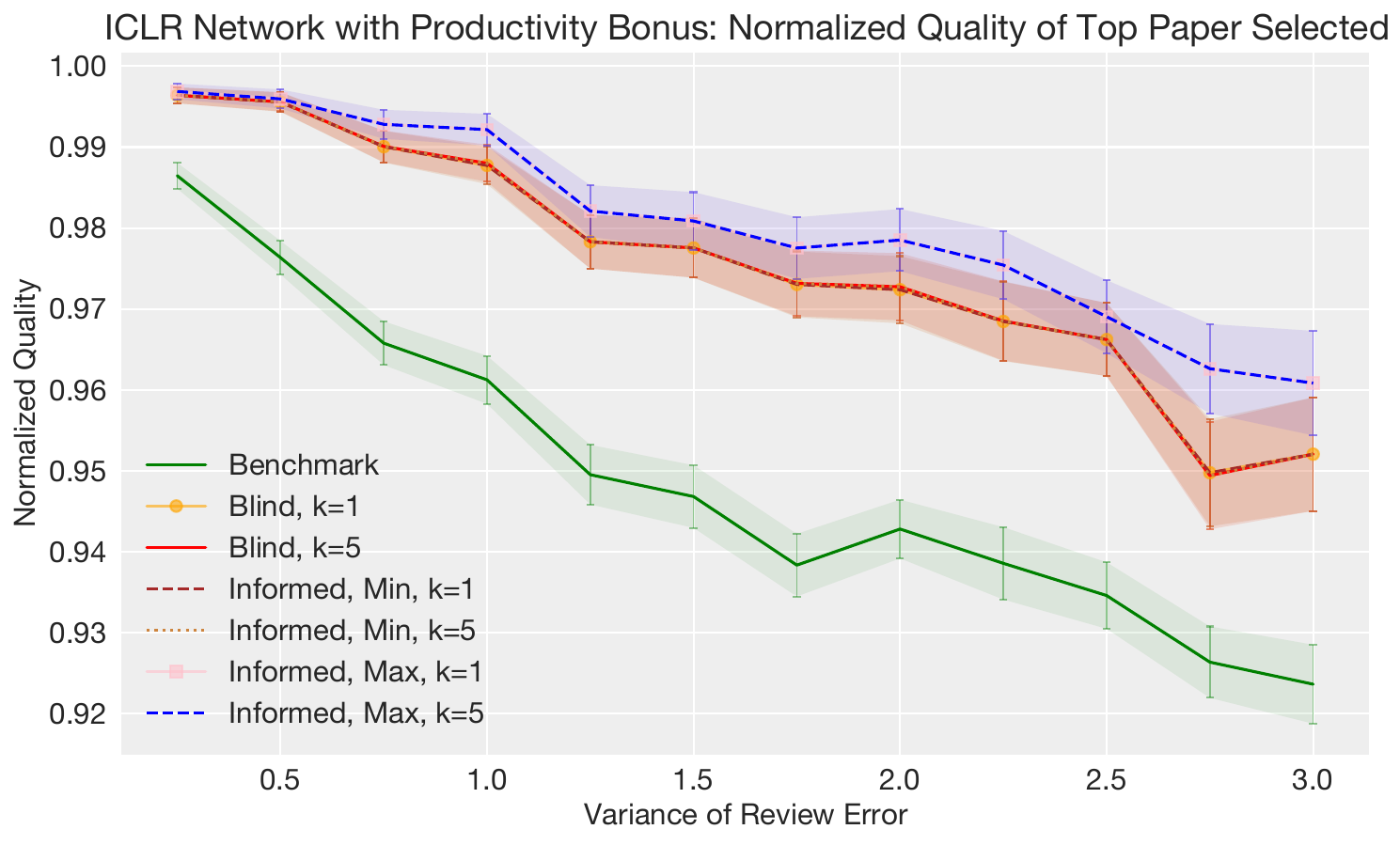}
    \caption{Comparison of Different Methods to select top paper, ICLR Network with Productivity Bonus}
    \label{fig:bestlog2}
  \end{minipage}
  \hfill 
  \begin{minipage}{.45\textwidth}
    \centering
    \includegraphics[width=1.2\linewidth]{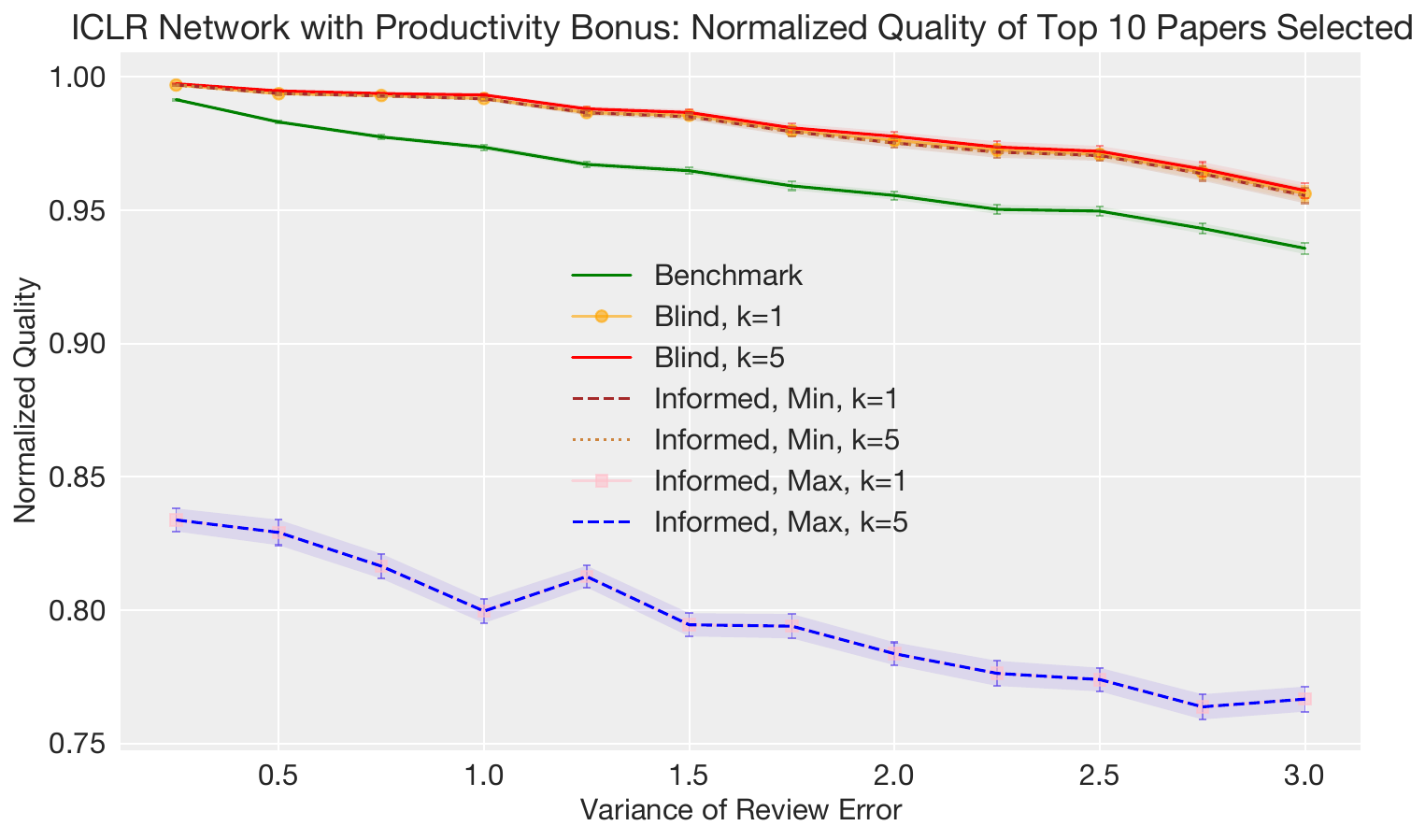}
    \caption{Comparison of Different Methods to select top 10 papers, ICLR Network with Productivity Bonus (see~\ref{note})}
    \label{fig:second12}
  \end{minipage}
\end{figure}

\begin{figure}[h]
 \centering
 \begin{minipage}{.45\textwidth}
   \centering
   \includegraphics[width=1.2\linewidth]{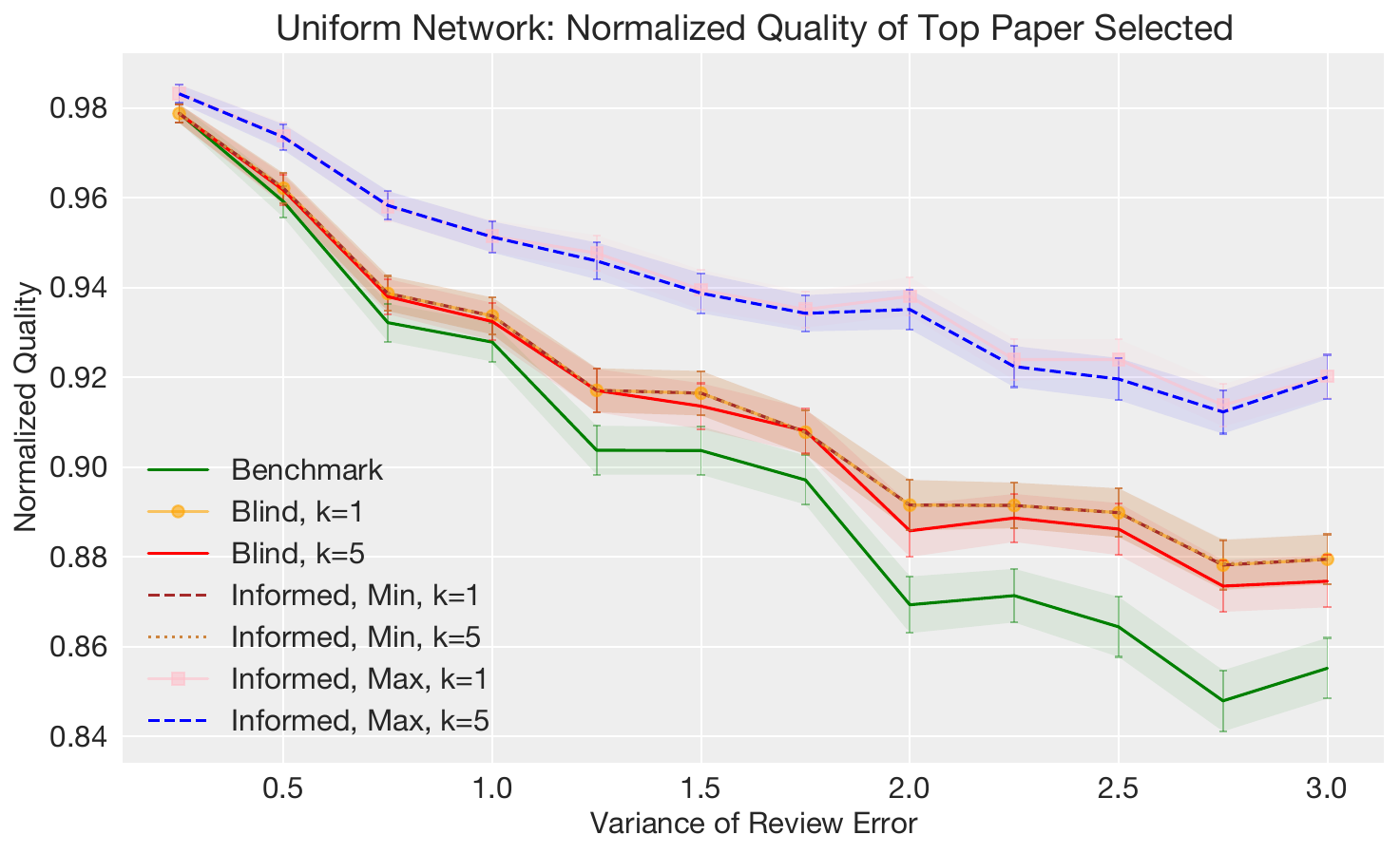}
   \caption{Comparison of Different Methods to select top paper, Uniform Network}
   \label{fig:bestlog2}
 \end{minipage}
 \hfill 
 \begin{minipage}{.45\textwidth}
   \centering
   \includegraphics[width=1.2\linewidth]{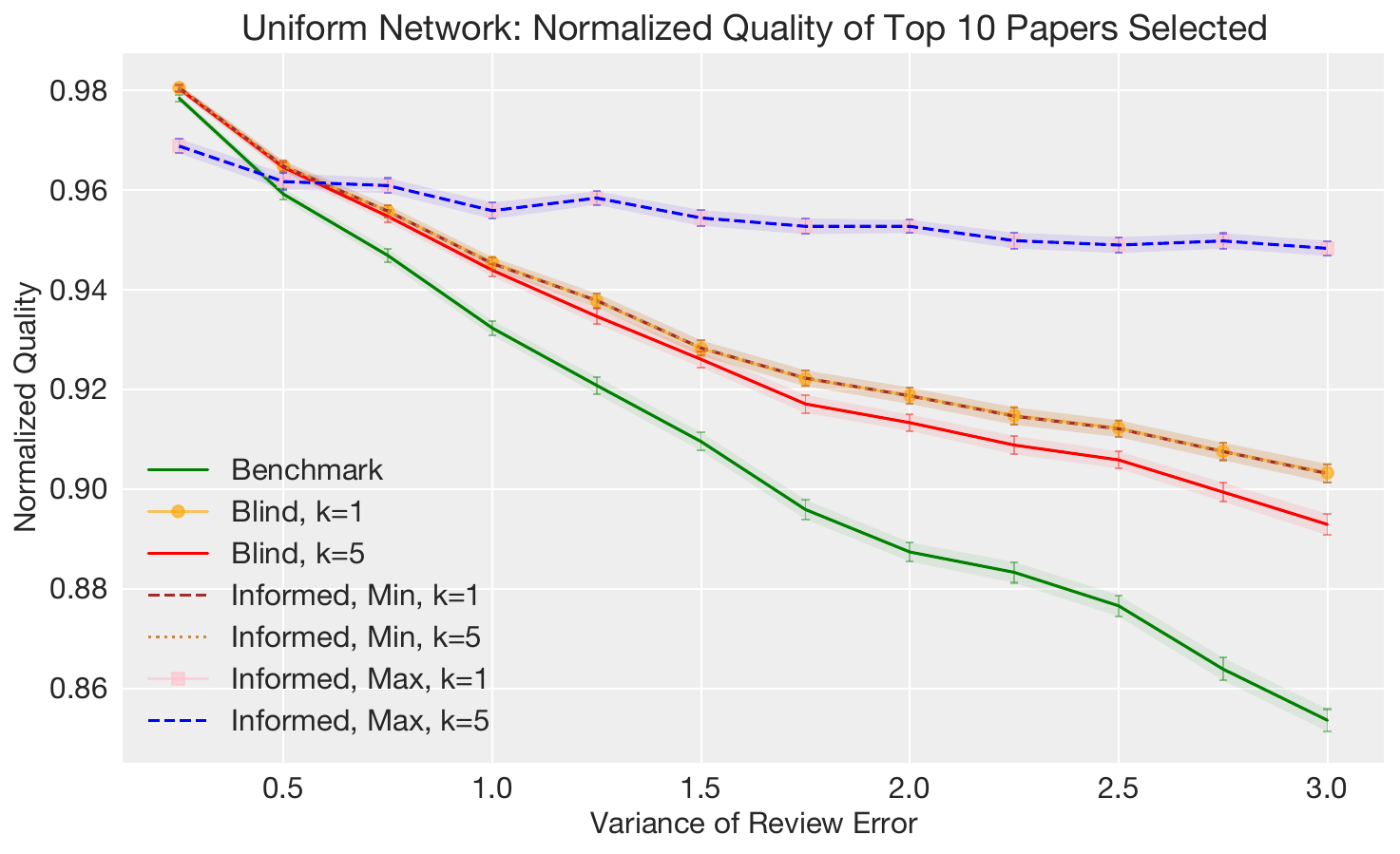}
      \caption{Comparison of Different Methods to select top 10 papers, Uniform Network}
   \label{fig:second12}
 \end{minipage}
\end{figure}

\begin{figure}[h] 
  \centering
  \begin{minipage}{.45\textwidth}
    \centering
    \includegraphics[width=1.2\linewidth]{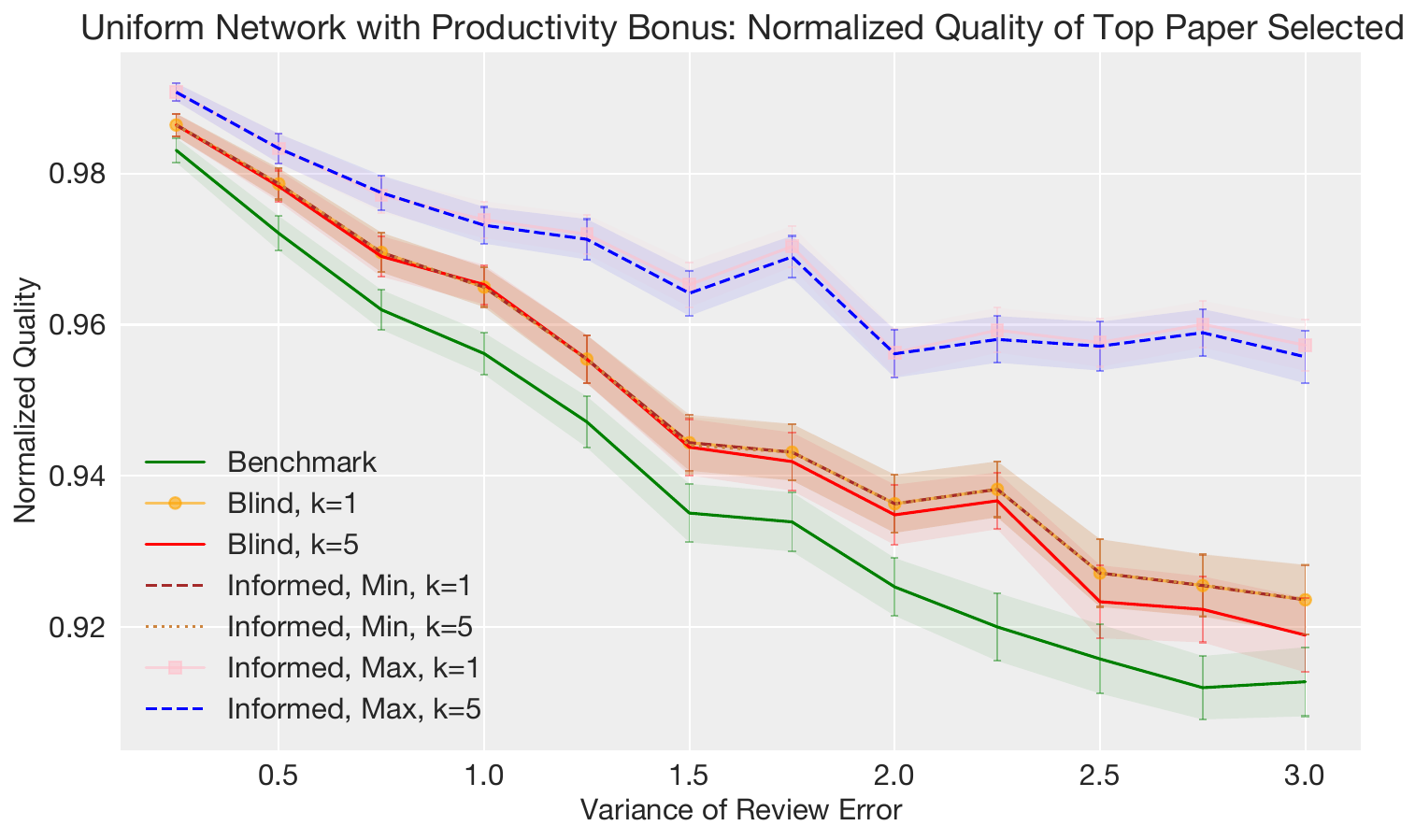}
       \caption{Comparison of Different Methods to select top paper, Uniform Network with Productivity Bonus}
    \label{fig:bestlog2}
  \end{minipage}
  \hfill 
  \begin{minipage}{.45\textwidth}
    \centering
    \includegraphics[width=1.2\linewidth]{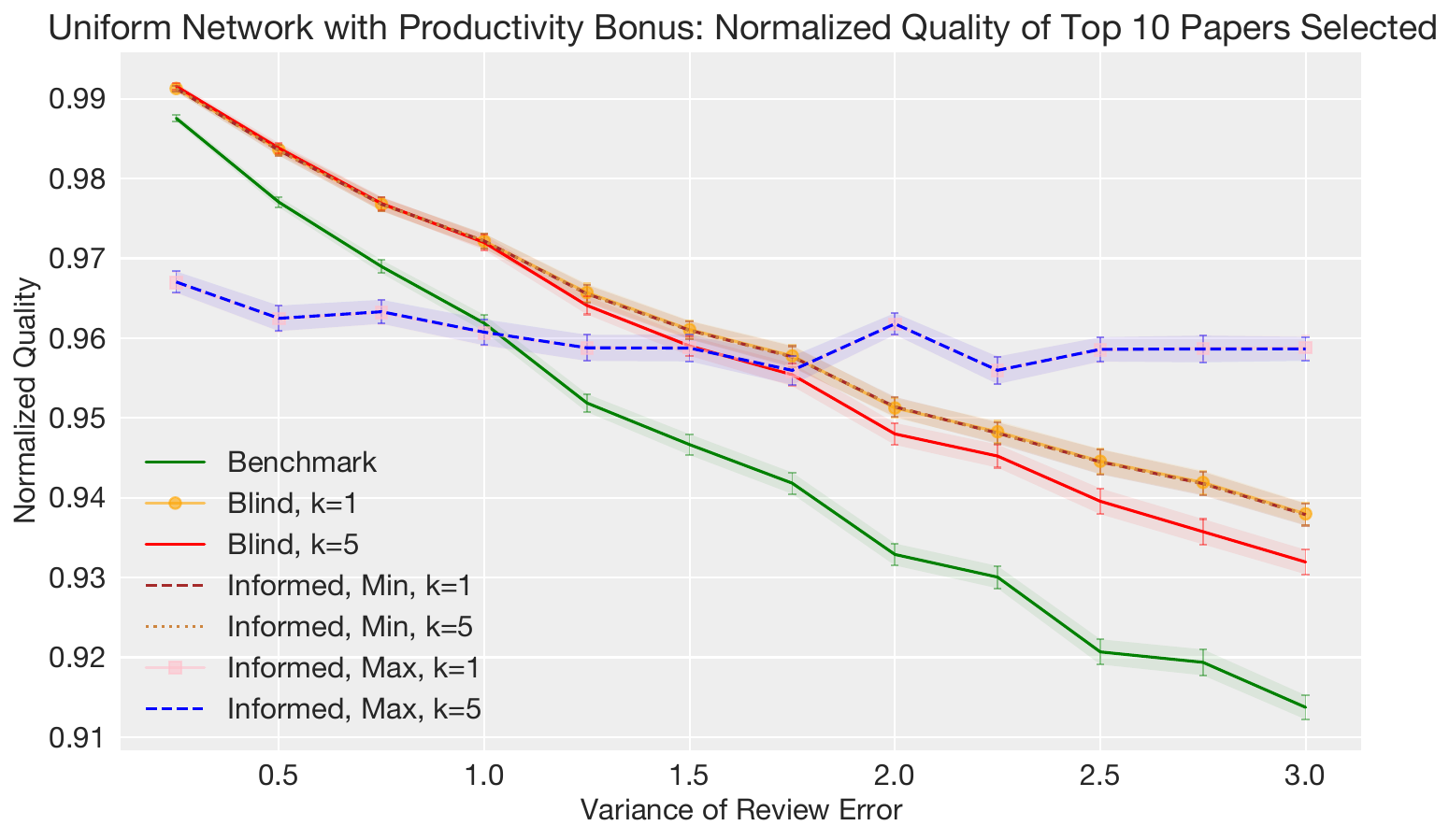}
    \caption{Comparison of Different Methods to select top 10 papers, Uniform Network with Productivity Bonus}
    \label{fig:second12}
  \end{minipage}
\end{figure}
\clearpage
 
\subsection{Example for the Edge Cases}\label{app:edge}
Assume we have three authors: A, B, and C. There are a total of 12 papers, numbered 1 through 12. Each paper is assigned a review score in a sequential order from 1 to 12, and additionally, there is a latent score for each paper in descending order starting from 12 down to 1.

The authorship of the papers is as follows:
\begin{itemize}
\item Author A: Papers \{1, 2, 3, 4, 5, 6\}
\item Author B: Papers \{5, 6, 7, 8, 9, 10\}
\item Author C: Papers \{9, 10, 11, 12\}
\end{itemize}

If the authors are aware of the latent scores, then Author A would rank paper 1 the highest and paper 6 the lowest. The same ranking principle applies to Authors B and C based on their respective papers.

Using a 1-partition greedy algorithm, we get the following partitioning of papers:
\begin{itemize}
\item Author A: Papers \{1, 2, 3, 4, 5, 6\}
\item Author B: Papers \{7, 8, 9, 10\}
\item Author C: Papers \{11, 12\}
\end{itemize}

After running the Isotonic Mechanism, we obtain the scores as follows:
\begin{itemize}
\item Author A: Scores \{3.5, 3.5, 3.5, 3.5, 3.5, 3.5\}
\item Author B: Scores \{8.5, 8.5, 8.5, 8.5\}
\item Author C: Scores \{11.5, 11.5\}
\end{itemize}

In the $k=2$ blind review case, the papers recommended by the authors along with their scores are:
\begin{itemize}
\item Author A recommends Papers with scores: 3.5, 3.5
\item Author B recommends Papers with scores: 3.5, 3.5
\item Author C recommends Papers with scores: 8.5, 8.5
\end{itemize}

Then the paper you finally choose is not the paper with the highest adjusted score.

\bibliographystyle{abbrv}
\bibliography{ref}

\end{document}